\documentclass[twoside,11pt]{article}
\usepackage{jair, natbib, rawfonts}
\usepackage{times}
\usepackage{helvet}
\usepackage{courier}
\usepackage{url}
\usepackage{graphicx} 
\usepackage{subfigure} 
\usepackage{algorithm}
\usepackage{algorithmic}
\usepackage{amsmath}[11pt]
\usepackage{amsmath, amsthm, amssymb}
\usepackage{amsthm}
\usepackage{amsfonts}[11pt]
\usepackage{mathtools}
\usepackage{bbm}
\newtheorem*{propositionOCO}{Proposition (OGD regret)}
\newtheorem{thm}{Theorem}
\newtheorem{cor}[thm]{Corollary}
\newtheorem{lem}[thm]{Lemma}
\newcommand{\bc}{\begin{center}}
\newcommand{\ec}{\end{center}}
\newcommand{\bq}{\begin{quote}}
\newcommand{\eq}{\end{quote}}

\newcommand{\be}{\begin{equation}}
\newcommand{\ee}{\end{equation}}
\newcommand{\beqa}{\begin{eqnarray*}}
\newcommand{\eeqa}{\end{eqnarray*}}
\newcommand{\beqn}{\begin{eqnarray}}
\newcommand{\eeqn}{\end{eqnarray}}
\newcommand{\bbibl}{\begin{thebibliography}{99}}
\newcommand{\ebibl}{\end{thebibliography}}

\newcommand{\ba}{\begin{array}}
\newcommand{\ea}{\end{array}}

\newcommand\ind[1]{\mathbf{1}\left[#1\right]}

\newtheorem*{defnNorm}{Definition (p $\rightarrow$ q norm)}
\DeclareMathOperator*{\argmax}{argmax}

\DeclareMathOperator*{\argsort}{argsort}

\DeclarePairedDelimiter\floor{\lfloor}{\rfloor}

\begin{document}

\title{Perceptron-like Algorithms for Online Learning to Rank}

\author{\name Sougata Chaudhuri \email sougata@umich.edu\\
	\addr Department of Statistics,\\
	University of Michigan, \\
	1085 S. University Ave., Ann Arbor, MI 48109, USA\\
       \name Ambuj Tewari \email tewaria@umich.edu \\
       \addr Department of Statistics, and\\
       Department of Electrical Engineering and Computer Science,\\
       University of Michigan, \\
       1085 S. University Ave., Ann Arbor, MI 48109, USA}


\maketitle

\begin{abstract}
Perceptron is a classic online algorithm for learning a classification function. In this paper, we provide a novel extension of the perceptron algorithm to the learning to rank problem in information retrieval. We consider popular listwise performance measures such as Normalized Discounted Cumulative Gain (NDCG) and Average Precision (AP). We propose a novel family of listwise, large margin ranking surrogates, which are adaptable to NDCG and AP measures and derive a perceptron-like algorithm using these surrogates. Exploiting a self-bounding property of the proposed surrogates, we provide a guarantee on the cumulative NDCG (or AP) induced loss incurred by our perceptron-like algorithm. We show that, if there exists a perfect oracle ranker which can correctly rank, with some margin, each instance in an online sequence, the cumulative NDCG (or AP) induced loss of perceptron algorithm on that sequence is bounded by a constant, irrespective of the length of the sequence. This result is a learning to rank analogue of Novikoff's convergence theorem for the classification perceptron. However, our perceptron like algorithm for learning to rank has two drawbacks. First, unlike classification perceptron, the prediction at each round depends on a learning rate parameter. Second, the perceptron loss bound does not match our established lower bound on the cumulative loss achievable by any deterministic online algorithm. We propose a second perceptron like algorithm which achieves the lower bound and is independent of the learning rate parameter. However, our second algorithm does not adapt to different ranking measures, does not possess the listwise property and does not perform well on real world datasets. Experiments on simulated datasets corroborate our theoretical results and demonstrate competitive performance on large industrial benchmark datasets.
\end{abstract}

\section{Introduction}
\label{introduction}

Learning to rank \citep{liu2011learning} is a supervised learning problem where the output space consists of \emph{rankings} of a set of objects. In the learning to rank problem that frequently arises in information retrieval, the objective is to rank \emph{documents} associated with a \emph{query}, in the order of the \emph{relevance} of the documents for the given query. The accuracy of a ranked list, given actual relevance scores of the documents, is measured by various ranking performance measures, such as \emph{Normalized Discounted Cumulative Gain} (NDCG) \citep{jarvelin2002} and \emph{Average Precision} (AP) \citep{baeza1999}.  Since optimization of ranking measures during the training phase is computationally intractable, ranking methods are often based on minimizing \emph{surrogate} losses that are easy to optimize. 


The historical importance of the perceptron algorithm in the classification literature is immense \citep{rosenblatt1958,freund1999}. 
Classically the perceptron algorithm was not linked to surrogate minimization but the modern perspective on perceptron is to interpret it as online gradient descent (OGD), during mistake rounds, on the hinge loss function \citep{shalev2011}. The hinge loss has special properties that allow one to establish bounds on the cumulative zero-one loss (viz., the total number of mistakes) in classification, without making any statistical assumptions on the data generating mechanism. Novikoff's celebrated result \citep{novikoff1962} about the perceptron says that, if there is a perfect linear classification function which can correctly classify, with some margin, every instance in an online sequence, then the total number of mistakes made by perceptron, on that sequence, is bounded. Moreover, unlike the standard OGD algorithm, the performance of perceptron is independent of learning rate parameter, which is of significant advantage due to not having to learn the optimal parameter value.


Our work provides a novel extension of the perceptron algorithm to the learning to rank setting with a focus on two listwise ranking measures, NDCG and AP. Listwise measures are so named because the quality of ranking function is judged on an entire list of document, associated with a query, usually with an emphasis to avoid errors near top of the ranked list. 
Specifically, we make the following contributions in this work.
\begin{itemize}
\item We develop a family of \emph{listwise} large margin ranking surrogates. The family consists of Lipschitz functions and is parameterized by a set of weight vectors that makes the surrogates adaptable to losses induced by performance measures NDCG and AP.  The family of surrogates is an extension of the hinge surrogate in classification that upper bounds the $0$-$1$ loss. The family of surrogates has a special self-bounding property: the norm of the gradient of a surrogate can be bounded by the surrogate loss itself.
\item We exploit the self bounding property of the surrogates to develop an online perceptron-like algorithm for learning to rank (Algorithm~\ref{alg:LA}). We provide bounds on the cumulative NDCG and AP induced losses (Theorem~\ref{theoryboundinperceptron}). We prove that, if there is a \emph{perfect} linear ranking function which can rank correctly, with some margin, every instance in an online sequence, our perceptron-like algorithm perfectly ranks all but a finite number of instances (Corollary~\ref{cor:margin}). This implies that the cumulative loss induced by NDCG or AP is bounded by a constant, and our result can be seen as an extension of the classification perceptron mistake bound (Theorem~\ref{theoryboundinclassificationperceptron}). The performance of our perceptron algorithm, however, is dependent on a learning rate parameter, which is a disadvantage over classification perceptron. Moreover, the bound depends linearly on the number of documents per query. In practice, during evaluation, NDCG is often cut off at a point which is much smaller than number of documents per query. In that scenario, we prove that the cumulative NDCG loss of our perceptron is upper bounded by a constant which is dependent only on the cut-off point. (Theorem~\ref{thm:k-dependence}).
\item We prove a lower bound, on the cumulative loss induced by NDCG or AP, that can be achieved by any deterministic online algorithm (Theorem~\ref{lower-bound}) under a separability assumption. The lower bound is \emph{independent} of the number of documents per query. We propose a second perceptron like algorithm (Algorithm~\ref{alg:OA}) which achieves the lower bound (Theorem~\ref{theoryboundinonlinealg}), with performance being independent of learning rate parameter. However, the surrogate on which the perceptron type algorithm operates is not listwise in nature and does not adapt to different performance measures. Thus, its empirical performance on real data is significantly worse than the first perceptron algorithm (Algorithm~\ref{alg:LA}).
\item We provide empirical results on simulated as well as large scale benchmark datasets and compare the performance of our perceptron algorithm with the online version of the widely used ListNet learning to rank algorithm \citep{Cao2007}. 
\end{itemize}

The rest of the paper is organized as follows. Section~\ref{probdef} provides formal definitions and notations related to the problem setting. Section~\ref{classification-perceptron} provides a review of perceptron for classification, including algorithm and theoretical analysis. Section~\ref{SLAM}  introduces the family of listwise large margin ranking surrogates, and contrasts our surrogates with a number of existing large margin ranking surrogates in literature. Section~\ref{perceptron} introduces the perceptron algorithm for learning to rank, and discusses various aspects of the algorithm and the associated theoretical guarantee. Section~\ref{minimax-bound} establishes a lower bound on NDCG/AP induced cumulative loss and introduces the second perceptron like algorithm. Section~\ref{related-work} compares our work with existing perceptron algorithms for ranking. Section~\ref{experiments} provides empirical results on simulated and large scale benchmark datasets.

\section{Problem Definition}
\label{probdef}
In learning to rank, we formally denote the input space as $\mathcal{X} \subseteq \mathbb{R}^{m \times d}$. Each input consists of $m$ rows of document-query features represented as $d$ dimensional vectors. Each input corresponds to a single query
and, therefore, the $m$ rows have features extracted from the same query but $m$ different documents. In practice $m$ changes from one input instance to another but we treat $m$ as a constant for ease of presentation. For $X \in \mathcal{X}$, $X= (x_1,\ldots,x_m)^{\top}$, where $x_i \in \mathbb{R}^d$ is the feature extracted from a query and the $i$th document associated with that query. The supervision space is $\mathcal{Y} \subseteq \{0,1,\ldots,n\}^m$, representing relevance score vectors. If $n=1$, the relevance vector is binary graded. For $n>1$, relevance vector is multi-graded. Thus, for $R \in \mathcal{Y}$, $R=(R_1,\ldots,R_m)^{\top}$, where $R_i$ denotes relevance of $i$th document to a given query. Hence, $R$ represents a vector and $R_i$, a scalar, denotes $i$th component of vector. Also, relevance vector generated at time $t$ is denoted $R_t$ with $i$th component denoted $R_{t,i}$.


The objective is to learn a ranking function which ranks the documents associated with a query in such a way that more relevant documents are placed ahead of less relevant ones.  The prevalent technique is to learn a scoring function and obtain a ranking by sorting the score vector in descending order. For $X \in \mathcal{X}$, a linear scoring function is $f_{w}(X)= X\cdot w= s^w \in \mathbb{R}^m$, where $w \in \mathbb{R}^d$. The quality of the learnt ranking function is evaluated on a test query using various performance measures. We use two of the most popular performance measures in our paper, viz. NDCG and AP. 

NDCG, cut off at $k \le m$ for a query with $m$ documents, with relevance vector $R$ and score vector $s$ induced by a ranking function, is defined as follows:
\begin{equation}
\label{eq:NDCGk}
\text{NDCG}_k(s,R) = \frac{1}{Z_k(R)}\sum_{i=1}^k G(R_{\pi_s(i)})D(i).
\end{equation}
Shorthand representation of $\text{NDCG}_k(s,R)$ is $\text{NDCG}_k$. Here, $G(r)= 2^r -1$, $D(i)= \frac{1}{\log_2{(i+1)}}$, $Z_k(R)= \underset {\pi \in S_m}{\max}\sum_{i=1}^k G(R_{\pi(i)})D(i)$. Further, $S_m$ represents the set of permutations over $m$ objects. $\pi_s = \argsort(s)$ is the permutation induced by sorting score vector $s$ in descending order (we use $\pi_s$ and $\argsort(s)$ interchangeably). A permutation $\pi$ gives a mapping from ranks to documents and $\pi^{-1}$ gives a mapping from documents to ranks. Thus, $\pi(i)=j$ means document $j$ is placed at position $i$ while $\pi^{-1}(i)=j$ means document $i$ is placed at position $j$. For $k=m$, we denote $\text{NDCG}_m(s,R)$ as $\text{NDCG}(s,R)$. 
The popular performance measure, Average Precision (AP), is defined only for binary relevance vector, i.e., each component can only take values in $\{0,1\}$:
\begin{equation}
\label{eq:AP}
\small
\text{AP}(s,R) = \frac{1}{r} \sum_{j:R_{\pi_s(j)}=1} \frac{\sum_{i \le j} \mathbbm{1}[R_{\pi_s(i)}=1]}{j}
\end{equation}
where $r= \|R\|_1$ is the total number of relevant documents.

All ranking performances measures are actually \emph{gains}. When we say ``NDCG induced loss", we mean a loss function that simply subtracts NDCG from its maximum possible value, which is $1$ (same for AP).

\section{Perceptron for Classification}
\label{classification-perceptron}
We will first briefly review the perceptron algorithm for classification, highlighting the modern viewpoint that it executes online gradient descent (OGD) \citep{zinkevich2003online} on hinge loss during mistake rounds and achieves a bound on total number of mistakes. This will allow us to directly compare and contrast our extension of perceptron to the learning to rank setting. For more details, we refer the reader to the survey written by \citet[Section 3.3]{shalev2011}.

In classification, an instance is of the form $x \in \mathbb{R}^d$ and corresponding supervision (label) is $y \in \{-1,1\}$. A linear classifier is a scoring function $g_w(\cdot)$, parameterized by $w \in \mathbb{R}^d$, producing score $g_w(x)= x \cdot w= s \in \mathbb{R}$. Classification of $x$ is obtained by using ``sign" predictor on $s$, i.e., $\text{sign}(s)$ $\in \{-1,1\}$. The loss is of the form: $\ell(w,(x,y))= \mathbbm{1}[\text{sign}(x \cdot w) \neq y]$. The hinge loss is defined as: $\phi(w, (x,y))= [1- y(x \cdot w)]_{+}$, where $[a]_+= \max \{0,a\}$. 

The perceptron algorithm operates on the loss $f_t(w)$, defined on a sequence of data $\{x_t,y_t\}_{t \ge 1}$, produced by an \emph{adaptive adversary} as follows:
\begin{equation}
\label{eq:perceptronfunction}
f_t(w)=\\
\left\{
	\begin{array}{ll}
		 [1- y_t(x_t \cdot w)]_+ & \mbox{if } \ell(w_t,(x_t,y_t))=1\\
		 0 & \mbox{if } \ell(w_t, (x_t,y_t))=0\\ 
           \end{array}
\right.
\end{equation}
where $w_t$ is the learner's move in round $t$. It is important to understand the concept of the loss $f_t(\cdot)$ and adaptive adversary here. An adaptive adversary is allowed to choose $f_t$ at round $t$ based on the moves of the perceptron algorithm (Algorithm~\ref{alg:PC}) upto that round. Once the learner fixes its choice $w_t$ at the end of step $t-1$, the adversary decides which function to play. It is either $[1- y_t(x_t \cdot w)]_{+}$ or 0, depending on whether  $\ell(w_t,(x_t,y_t))$ is 1 or 0 respectively. Notice that $f_t(w)$ is convex in both cases.


The perceptron updates a classifier $g_{w_t}(\cdot)$ (effectively updates $w_t$), in an online fashion. The update occurs by application of OGD on the sequence of functions $f_t(w)$ in the following way: perceptron initializes $w_1=\vec{0}$ and uses update rule $w_{t+1}= w_t - \eta z_t$, where $z_t \in \partial{f_t}(w_t)$ ($z_t$ is a subgradient) and $\eta$ is the learning rate (the importance of $\eta$ will be discussed at the end of the section). If $\ell(w_t, (x_t,y_t))=0$, then $f_t(w_t)=0$; hence $z_t= \vec{0}$. Otherwise, $z_t= -y_t x_t \in  \partial{f_t}(w_t)$.  Thus,
\begin{equation}
\label{eq:perceptronclassification}
w_{t+1} = \\
\left\{
	\begin{array}{ll}
		w_t & \mbox{if } \ell(w_t,(x_t,y_t))=0\\
		w_t + \eta y_t x_t  & \mbox{if } \ell(w_t, (x_t,y_t))=1.\\ 
           \end{array}
\right.
\end{equation}

The perceptron algorithm for classification is described below:

\floatstyle{ruled}
\newfloat{algorithm}{htbp}{loa}
\floatname{algorithm}{Algorithm}
\begin{algorithm}
\caption{Perceptron Algorithm for Classification}
\label{alg:PC}
\begin{tabbing}
Learning rate $\eta>0$, $w_1=\mathbf{0} \in \mathbb{R}^d$.\\
For \=$t=1$ to $T$ \\
\> Receive $x_t$.\\
\> Predict $p_t= \text{sign}(x_t \cdot w_t)$.\\
\> Receive $y_t$ \\
\> {\bf If} $\ell(w_t,(x_t,y_t)) \neq 0$ \= \\
\> \ \ $w_{t+1} =w_t + \eta y_t x_t$ \\
\> {\bf else}\\
\> \ \ $w_{t+1} =w_t $\\
{\bf End For}
\end{tabbing}
\end{algorithm}

\begin{thm}
\label{theoryboundinclassificationperceptron}
Suppose that the perceptron for classification algorithm runs on an online sequence of data $\{(x_1,y_1), \ldots, (x_T,y_T)\}$ and let $R_x= \max_t \|x_t\|_2$. Let $f_t(\cdot)$ be defined as in Eq.~\ref{eq:perceptronfunction}. For all $u \in \mathbb{R}^d$ and setting $\eta= \frac{\|u\|_2}{R_x \sum_{t=1}^T \ell(w_t, (x_t,y_t))}$, the perceptron mistake bound is:
\begin{equation}
\label{eq:classification-perceptronbound}
\sum_{t=1}^T \ell(w_t, (x_t,y_t)) \le \sum_{t=1}^T f_t(u) + R_x \|u\|_2 \sqrt{\sum_{t=1}^T f_t(u)} + R_x^2 \|u\|_2^2
\end{equation}
In the special case where there exists  $u$ s.t. $f_t(u)=0$, $\forall \ t$, we have
\begin{equation}
\label{eq:classification-perceptronconstant}
\forall \ T, \ \sum_{t=1}^T \ell(w_t, (x_t,y_t)) \le  R_x^2 \|u\|_2^2
\end{equation}
\end{thm}

As can be clearly seen from Eq.~\ref{eq:classification-perceptronbound}, the cumulative loss bound (i.e., total number of mistakes over $T$ rounds) is upper bounded in terms of the cumulative sum of the functions $f_t(\cdot)$. In the special case where there exists a perfect linear classifier with margin, Eq.~\ref{eq:classification-perceptronconstant} shows that the total number of mistakes is bounded, regardless of the number of instances. 

One drawback of the bound in Eq.~\ref{eq:classification-perceptronconstant} is that the concept of margin is not explicit, i.e., it is hidden in the norm of the parameter of the perfect classifier ($\|u\|_2$). Let us assume that there is a linear classifier parameterized by a unit norm vector $u_\star$, such that all instances $x_t$ are not only correctly classified, but correctly classified with a \emph{margin} $\gamma$, defined as:
\begin{equation}
\label{eq:classification-margin}
y_t (x_t \cdot u_\star) \ge \gamma, \ \forall \ T
\end{equation}
It is easy to see that the scaled vector $u = u_\star/\gamma$, whose norm is $1/\gamma^2$, will satisfy $f_t(u)=0$ for all $t$. Therefore, we have following corollary.
\begin{cor}
If the margin condition~\eqref{eq:classification-margin} holds, then total number of mistakes is upper bounded by $\dfrac{R_x^2}{\gamma^2} $, a bound independent of the number of instances in the online sequence.
\end{cor}

{\bf Importance of learning rate parameter $\eta$}: The prediction at round $t$ is $p_t= \text{sign}(x_t \cdot w_t)$. Let $\mathcal{M}_t$ indicate the rounds, up to time point  $t-1$, where perceptron made a mistake. Starting from $w_1= \vec{0}$, unraveling $w_t$, we get $p_t = \text{sign}(\sum_{i \in \mathcal{M}_t} \eta\  x_t \cdot (y_i x_i) )$. It can be easily seen that $p_t$ is invariant to value of $\eta$, for $\eta>0$. Hence, the actual performance of the perceptron algorithm (in terms of total number of mistakes) is independent of learning rate $\eta$ and thus, $\eta=1$ can be fixed from the beginning of the algorithm. The reason for including $\eta$ in the algorithm is that in the subsequent analysis (Theorem~\ref{theoryboundinclassificationperceptron}), the perceptron loss bound uses standard regret analysis of OGD, where the optimal regret bound is established by optimizing over learning rate $\eta$. So, though the performance is actually independent of $\eta$, optimization over $\eta$ is necessary to establish the \emph {optimal theoretical upper bound on the loss}.

\section {A Novel Family of Listwise Surrogates}
\label{SLAM}
We define the novel SLAM family of loss functions: these are Surrogate, Large margin, Listwise and Lipschitz losses, Adaptable to multiple performance measures, and can handle Multiple graded relevance.
For score vector $s \in \mathbb{R}^m$, and relevance vector $R \in \mathcal{Y}$, the family of convex loss functions is defined as:
\begin{equation}
\label{eq:lossdef}
\begin{aligned}
\phi^v_{SLAM}(s,R) &= \min_{\delta \in \mathbb{R}^m} \sum_{i=1}^m v_i \delta_i \\
 \text{s.t.} \ \ \ \  \delta_i \ge 0, & \ \forall \ i, 
\ \ \ \ s_i +\delta_i \ge \Delta + s_j, \ \text{if} \  R_i > R_j, \ \ \forall \ i,j.
\end{aligned}
\end{equation}
The constant $\Delta$ denotes margin and $v = (v_1,\ldots,v_m)$ is an element-wise non-negative weight vector. Different vectors $v$, to be defined later, yield different members of the SLAM family. Though $\Delta$ can be varied for empirical purposes, we fix $\Delta=1$ for our analysis. The intuition behind the loss setting is that scores associated with more relevant documents should be higher, with a margin, than scores associated with less relevant documents. The weights decide how much weight to put on the errors.

The following reformulation of $\phi^v_{SLAM}(s,R)$ will be useful in later derivations. 
{\small
\begin{equation}
\label{eq:theoreticalloss}
\begin{split}
\sum_{i=1}^{m} v_i \  \max(0,\underset{j=1,\ldots,m}{ \max}\{\mathbbm{1}(R_i>R_j)(1+ s_j -s_i)\}) \ .
\end{split}
\end{equation}}

\begin{lem}
\label{convexity}
For any relevance vector $R$, the function $\phi^v_{SLAM}(\cdot,R)$ is \emph{convex}. 
\end{lem}

\begin{proof} 
Claim is obvious from the representation given in Eq.~\ref{eq:theoreticalloss}.
%
%
%
\end{proof}

\subsection{Weight Vectors Parameterizing the SLAM Family}
\label{upperbounds}
As we stated after Eq.~\ref{eq:lossdef}, different weight vectors lead to different members of the SLAM family. The weight vectors play a crucial role in the subsequent theoretical analysis. We will provide two weight vectors, $v^{\text{AP}}$ and $v^{\text{NDCG}}$, that result in upper bounds for AP and NDCG induced losses respectively. Later, we will discuss the necessity of choosing such weight vectors.

Since the losses in SLAM family is calculated with the knowledge of the relevance vector $R$, for ease of subsequent derivations, we can assume, without loss of generality, that documents are sorted according to their relevance levels. Thus, we assume that $R_1 \ge R_2 \ge \ldots \ge R_m$, where $R_i$ is the relevance of document $i$. Note that both $v^{\text{AP}}$ and $v^{\text{NDCG}}$ depend on the relevance vector $R$ but we hide that dependence in the notation to reduce clutter.


{\bf Weight vector for AP loss}: Let $R \in \mathbb{R}^m$ be a binary relevance vector. Let $r$ be the number of relevant documents (thus, $R_1=R_2=\ldots=R_r=1$ and $R_{r+1}=\ldots=R_m=0$). We define vector $v^{\text{AP}} \in \mathbb{R}^m$ as 
\begin{equation}
\label{eq:APweights}
v^{\text{AP}}_i = \\
\left\{
	\begin{array}{ll}
		\frac{1}{r}  & \mbox{if } i=1,2,\ldots,r\\
		0  & \mbox{if } i=r+1,\ldots,m .\\ 
           \end{array}
\right.
\end{equation}

{\bf Weight vector for NDCG loss}: For a given relevance vector $R \in \mathbb{R}^m$, we define vector $v^{\text{NDCG}} \in \mathbb{R}^m$ as 
\begin{equation}
\label{eq:ndcgweights}
\begin{split}
v^{\text{NDCG}}_i =   \frac{G(R_i)D(i)}{Z(R)}, \ i=1,\ldots,m .
\end{split}
\end{equation}
{\bf Note}: Both weights ensure that $v_1 \ge v_2 \ge \ldots \ge v_m$ (since $R_1 \ge R_2 \ge \ldots \ge R_m$).
Using the weight vectors, we have the following upper bounds.
\begin{thm}
\label{eq:upperbound1}
Let $v^{\text{AP}} \in \mathbb{R}^m$ and $v^{\text{NDCG}} \in \mathbb{R}^m$ be the weight vectors as defined in Eq.~\eqref{eq:APweights} and Eq.~\eqref{eq:ndcgweights} respectively. Let $\text{AP}(s,R)$  and $\text{NDCG}(s,R)$ be the AP value and NDCG value determined by  relevance vector $R \in \mathbb{R}^m$ and score vector $s \in \mathbb{R}^m$. Then, the following inequalities hold, $\forall \ s$, $\forall \ R$
\begin{equation}
\begin{aligned}
&\phi^{v^{\text{AP}}}_{\text{SLAM}}(s,R) \ge 1-\text{AP}(s,R)\\ 
&\phi^{v^{\text{NDCG}}}_{\text{SLAM}}(s,R) \ge 1-\text{NDCG}(s,R) \ .
\end{aligned}
\end{equation}
\end{thm}

The proof of the theorem is in Appendix~\ref{App:AppendixA}.

\subsection{Properties of SLAM Family and Upper Bounds}
We discuss some of the properties of SLAM family and related upper bounds.
\label{properties}
{\bf Listwise Nature of SLAM Family}: The critical property for a surrogate to be considered \emph{listwise} is that the loss must be calculated over the entire list of documents as a whole, with errors at the top penalized  more than errors at the bottom. Since perfect ranking places the most relevant documents at top, errors corresponding to most relevant documents should be penalized more in SLAM in order to be considered a listwise family. Both $v^{\text{NDCG}}$ and $v^{\text{AP}}$ have the property that the more relevant documents get more weight.\\
{\bf Upper Bounds on NDCG and AP}: By Theorem~\ref{eq:upperbound1}, the weight vectors make losses in SLAM family upper bounds on NDCG and AP induced losses. The SLAM loss family is analogous to the hinge loss in classification. Similar to hinge loss, the surrogate losses of SLAM family are $0$ when the predicted scores respect the relevance labels (with some margin).  The upper bound property will be crucial in deriving guarantees for a perceptron-like algorithm in learning to rank. Like hinge loss, the upper bounds can possibly be loose in some cases, but, as we show next, the upper bounding weights make SLAM family Lipschitz continuous with a small Lipschitz constant. This naturally restricts SLAM losses from growing too quickly. Empirically, we will show that the perceptron developed based on the SLAM family produce competitive performance on large scale industrial datasets. Along with the theory, the empirical performance supports the fact that upper bounds are quite meaningful.\\
{\bf Lipschitz Continuity of SLAM}: Lipschitz continuity of an arbitrary loss, $\ell(s,R)$ w.r.t. $s$ in ${\ell}_2$ norm, means that there is a constant $L_2$ such that $|\ell(s_1,R)- \ell(s_2,R)| \le L_2 \|s_1-s_2\|_2$, for all $s_1, s_2\in\mathbb{R}^m$. By duality, it follows that $L_2 \ge \underset{s}{\sup}\ \|\nabla_{s}\ell(s,R)\|_2$.  We calculate $L_2$ as follows:\\ 
Let $ b_{ij}= \{\mathbbm{1}(R_i>R_j)(1 +s_j -s_i)\}$. The sub-gradient of $\phi^v_{SLAM}$, w.r.t. to $s$, from Eq. \eqref{eq:theoreticalloss}, is: $\nabla_{s}{\phi^v_{SLAM}(s,R)} =  \sum_{i=1} ^ {m} v_i\ a^i $, where 
\begin{equation}
\label{eq:grad}
a^i = \\
\left\{
	\begin{array}{ll}
		\mathbf{0} \in \mathbb{R}^m   & \mbox{if } \underset{j=1,\ldots,m}{\max}b_{ij} \le 0\\
		\mathbf{e}_k - \mathbf{e}_i \in \mathbb{R}^m  & \text{otherwise, with } k=  \underset{j=1,\ldots, m}{\argmax}\ {b_{ij}}\\ 
                      
	\end{array}
\right.
\end{equation}
and $\mathbf{e}_i$ is a standard basis vector along coordinate $i$.

Since $\|a^i\|_1 \le 2$, it is easy to see that $\|\nabla_s{\phi^v_{SLAM}(s,R)}\|_{1} \le 2\sum_{i=1}^m v_i$. Since $\ell_1$ norm dominates $\ell_2$ norm,  $\phi^{v}_{SLAM}(s,R)$  is Lipschitz continuous in $\ell_2$ norm whenever we can bound  $\sum_{i=1}^m v_i$. It is easy to check that $\sum_{i=1}^m v^{\text{AP}}_i = 1$ and $\sum_{i=1}^m v^{\text{NDCG}}_i = 1$. Hence, $v^{\text{NDCG}}$ and $v^{\text{AP}}$ induce Lipschitz continuous surrogates, with Lipschitz constant at most 2.\\
{\bf Comparison with Surrogates Derived from Structured Prediction Framework}: We briefly highlight the difference between SLAM and listwise surrogates obtained from the structured prediction framework \citep{chapelle2007,yue2007,chakrabarti2008}.  Structured prediction for ranking models assume that the supervision space is the space of \emph{full rankings} of a document list. Usually a large number of full rankings are compatible with a relevance vector, in which case the relevance vector is arbitrarily mapped to a full ranking. In fact, here is a quote from one of the relevant papers \citep{chapelle2007}, \emph{``It is often the case that this $y_q$ is not unique and we simply take of one of them at random"} ($y_q$ refers to a correct full ranking pertaining to query $q$). Thus, all but one correct full ranking will yield a loss. In contrast, in SLAM, documents with same relevance level are essentially exchangeable (see Eq.~\eqref{eq:theoreticalloss}). Thus, our assumption that documents are sorted according to relevance during design of weight vectors is without arbitrariness, and there will be no change in the amount of loss when documents within same relevance class are compared. 

\section{Perceptron-like Algorithms}
\label{perceptron}

We present a perceptron-like algorithm for learning a ranking function in an online setting, using the SLAM family. Since our proposed perceptron like algorithm works for both NDCG and AP induced losses, for derivation purposes, we denote a performance measure induced loss as \emph{RankingMeasureLoss} (RML). Thus, RML can be NDCG induced loss or AP induced loss. 

{\bf Informal Definition}: The algorithm works as follows. At time $t$, the learner maintains a linear ranking function, parameterized by $w_t$. The learner receives $X_t$, which is the document list retrieved for query $q_t$ and ranks it. Then the ground truth relevance vector $R_t$ is received and ranking function updated according to the perceptron rule.


Let $ b_{ij}= \{\mathbbm{1}(R_i>R_j)(1 +s_j -s_i)\}$. For subsequent ease of derivations, we write SLAM loss from Eq.~\eqref{eq:theoreticalloss} as: $\phi_{SLAM}^v (s^w,R) =  \sum_{i=1} ^ {m} v_i\ c_i,$, where
\begin{equation}
\label{eq:surrogateinperceptron}
\begin{aligned}
& c_i= 
\left\{
	\begin{array}{ll}
		0   & \mbox{if } \underset{j=1,\ldots,m}{\max} b_{ij}\le 0\\
		1 + s^w_k -s^w_i \in \mathbb{R}  & \text{otherwise}\\ 
                      & k=  \underset{j=1,\ldots, m}{\argmax}\ {b_{ij}}.\\
	\end{array}
\right. \\
\end{aligned}
\end{equation}
and $s^w= Xw \in \mathbb{R}^m$.

Like classification perceptron, our perceptron-like algorithm operates on the loss $f_t(w)$, defined on a sequence of data $\{X_t,R_t\}_{t \ge 1}$, produced by an \emph{adaptive adversary} (i.e., an adversary who can see the learner's move before making its move) as follows:

\begin{equation}
\label{eq:functioninperceptron}
\begin{aligned}
& f_t(w)= 
\left\{
	\begin{array}{ll}
		\phi_{SLAM}^{v_t} (s^{w}_t,R_t)  & \mbox{if } \text{RML}(s^{w_t}_t,R_t) \neq 0\\
		0  & \mbox{if } \text{RML}(s^{w_t}_t,R_t) = 0\\                       
	\end{array}
\right. \\
\end{aligned}
\end{equation}
Here, $s^{w}_t=X_tw$ and $v_t=v_t^{\text{NDCG}}$ or $v_t^{\text{AP}}$ depending on whether RML is NDCG or AP induced loss. Since weight vector $v$ depends on relevance vector $R$ (Eq.~\eqref{eq:APweights}, \eqref{eq:ndcgweights}), the subscript $t$ in $v_t$ denotes the dependence on $R_t$.  Moreover, $w_t$ is the parameter produced by our perceptron (Algorithm~\ref{alg:LA}) at the end of step $t-1$, with the adaptive adversary being influenced by the move of perceptron (recall Eq.~\ref{eq:perceptronfunction} and discussion thereafter).

It is clear from Theorem.\ \ref{eq:upperbound1} and Eq.\ \eqref{eq:functioninperceptron} that $f_t(w_t) \ge \text{RML}(s^{w_t}_t,R_t)$.
It should also be noted that  that $f_t(\cdot)$ is convex in either of the two cases. 
Thus, we can run the online gradient descent (OGD) algorithm \citep{zinkevich2003online} to learn the sequence of parameters $w_t$, starting with $w_1=\mathbf{0}$. The OGD update rule, $w_{t+1}= w_t - \eta z_t$, for some $z_t \in \partial{f_t}(w_t)$ and step size $\eta$, requires a subgradient $z_t$ that, in our case, is computed as follows.
When $\text{RML}(s^{w_t}_t,R_t)=0$, we have $z_t=\mathbf{0} \in \mathbb{R}^d$.
When $\text{RML}(s^{w_t}_t,R_t) \neq 0$, we have
\begin{equation}
\label{eq:gradientinperceptron}
\begin{aligned}
& z_t=  X^{\top}_t \left(\sum_{i=1} ^ {m} v_{t,i}\ a_{t,i} \right) \in \mathbb{R}^d, \\
& a_{t,i}= 
\left\{
	\begin{array}{ll}
		\mathbf{0} \in \mathbb{R}^m   & \mbox{if } c^t_i = 0\\
		\mathbf{e}_k - \mathbf{e}_i \in \mathbb{R}^m  & \mbox{if }  c^t_i \neq 0\\ 
                     
	\end{array}	
\right.
\end{aligned}
\end{equation}
where $\mathbf{e}_k$ is the standard basis vector along coordinate $k$ and $c^t_i \in \mathbb{R}$  is as defined in Eq.~\eqref{eq:surrogateinperceptron} (with $s^{w}= s_t^{w_t}= X_t w_t$).


We now obtain a perceptron-like algorithm for the learning to rank problem.
\floatstyle{ruled}
\newfloat{algorithm}{htbp}{loa}
\floatname{algorithm}{Algorithm}
\begin{algorithm}
\caption{Perceptron Algorithm for Learning to Rank}
\label{alg:LA}
\begin{tabbing}
Learning rate $\eta>0$, $w_1=\mathbf{0} \in \mathbb{R}^d$.\\
For \=$t=1$ to $T$ \\
\> Receive $X_t$ (document list for query $q_t$). \\
\> Set $s^{w_t}_t = X_tw_t$ \ , predicted ranking output $p_t$= $\argsort(s^{w_t}_t)$.\\
\> Receive $R_t$ \\
\> {\bf If} $\text{RML}(s^{w_t}_t, R_t) \neq 0$ \= $\qquad$// Note:  $\text{RML}(s^{w_t}_t, R_t)=  \text{RML}(\argsort(s^{w_t}_t), R_t)$ \\ 
\> \ \ $w_{t+1} =w_t - \eta z_t$ \> $\qquad$// $z_t$ is defined in Eq.~\eqref{eq:gradientinperceptron} \\
\> {\bf else}\\
\> \ \ $w_{t+1} =w_t $\\
{\bf End For}
\end{tabbing}
\end{algorithm}

\subsection{Bound on Cumulative Loss}
\label{theorybounds}

We provide a theoretical bound on the cumulative loss (as measured by RML) of perceptron for the learning to rank problem. The technique uses regret analysis of online convex optimization algorithms.
We state the standard OGD bound used to get our main theorem \citep{zinkevich2003online}. An important thing to remember is that OGD guarantee holds for convex functions played by an adaptive adversary, which is important for an OGD based analysis of the perceptron algorithm.
\begin{propositionOCO}
\label{regretboundinperceptron}
Let $f_t$ be a sequence of convex functions. The update rule of function parameter is $w_{t+1}= w_t -\eta z_t$, where $z_t \in \partial f_t(w_t)$. Then for any $w \in \mathbb{R}^d$, the following regret bound holds after $T$ rounds,
\begin{equation}
\label{eq:OGDregret}
\sum_{t=1}^T f_t(w_ t)\ -\sum_{t=1}^T f_t(w) \ \le \ \frac{\|w\|_2^2}{2\eta} + \frac{\eta}{2} \sum_{t=1}^T \|z_t\|_2^2 .
\end{equation}
\end{propositionOCO}


We first control the norm of the subgradient $z_t$, defined in Eq.\ \eqref{eq:gradientinperceptron}. To do this, we will need to use the $p \rightarrow q$ norm of matrix.
\begin{defnNorm}
\label{pqnorm}
Let $A \in \mathbb{R}^{m \times n}$ be a matrix. The $p \rightarrow q$ norm of A is:
\begin{equation*}
\|A\|_{p \rightarrow q} = \max_{ v \neq 0} \frac{\|Av\|_q}{\|v\|_p}
\end{equation*}
\end{defnNorm}

\begin{lem}
\label{gradientboundinperceptron}
Let $R_X$ be the bound on the maximum $\ell_2$ norm of the feature vectors representing the documents. Let $v_{t,\mathrm{max}}= \underset{i,j}{\max}\{\frac{v_{t,i}}{v_{t,j}}\}, \ \forall \ i,j$ with $v_{t,i}>0,\  v_{t,j}>0$, and $m$ be bound on number of documents per query. Then we have the following $\ell_2$ norm bound, 
\begin{equation}
\forall \ t,\ \|z_t\|_2^2 \le 4 \ m\  R_X^2 \ v_{t,\mathrm{max}}\  f_t(w_t) \ .
\end{equation}
\end{lem}

\begin{proof}
For a mistake round $t$, we have $z_t = X^{\top}_t(\sum_{i=1}^m v_{t,i}a_{t,i})$ from Eq.~\eqref{eq:gradientinperceptron} . \\
{\bf 1st bound for $z_t$}:
\begin{equation*}
\begin{aligned}
\|X^{\top}_t(\sum_{i=1}^m v_{t,i}\ a_{t,i})\|_2 \le \|X^{\top}_t\|_{1\rightarrow 2}\|\sum v_{t,i} \ a_{t,i}\|_1 \le 2R_X \sum v_{t,i} = 2R_X.
\end{aligned}
\end{equation*}
The first inequality uses the $1 \rightarrow 2$ norm and last inequality holds because $\sum_{i=1}^m v^{NDCG}_i = 1$ and $\sum_{i=1}^m v^{AP}_i = 1$.\\\\
{\bf 2nd bound for $z_t$} (The self-bounding property of SLAM is being used here, to bound the norm of gradient by loss itself):\\
We note that in a mistake round, $\text{RML}(s_t^{w_t},R_t) \neq 0$. Thus, there is at least 1 pair of documents whose ranks are inconsistent with their relevance levels. Mathematically, 
\begin{equation*} 
\exists \ i', k' \ \text{s.t.} \ R_{t,i'} > R_{t,k'}, \  s^{w_t}_{t,i'}<s^{w_t}_{t,k'}. 
\end{equation*}
Now, $\phi_{SLAM}^{v_t}(s^{w_t}_t, R_t)= \sum v_{t,i}\  c^t_i$ (Eq.~\eqref{eq:surrogateinperceptron} ). For $(i',k')$, we have $c^t_{i'} \ge 1 +s^{w_t}_{t,k'} - s^{w_t}_{t,i'}>1$.\\

Since $R_{t,i'}>R_{t,k'}$, document $i'$ has strictly greater than minimum possible relevance, i.e., $R_{t,i'}>0$. By our calculations of weight vector $v$ for both NDCG and AP, we have $v_{t,i'}>0$.\\

Thus, by definition, $v_{t,\mathrm{max}} \ge 1$ (since $v_{t,i'}>0$ and $\frac{v_{t,i'}}{v_{t,i'}}=1$ and $v_{t,\mathrm{max}}= \underset{i,j}{\max}\{\frac{v_{t,i}}{v_{t,j}}\}, \ \forall \ i,j$ with $v_{t,i}>0,\  v_{t,j}>0$).\\

Then, $\forall \ i$, $v_{t,i} \le \ v_{t,\mathrm{max}}\ \cdot v_{t,i'} \le\  v_{t,\mathrm{max}}\ \cdot v_{t,i'}\ \cdot c^t_{i'}$. Thus, we have:
$$\sum_{i=1}^m v_{t,i} \le m\  v_{t,\mathrm{max}}\ v_{t,i'} \ c^t_{i'} \le \  m\  v_{t,\mathrm{max}}(\sum_{i=1}^m v_{t,i} \ c^t_i) = m\  v_{t,\mathrm{max}}\ \phi_{SLAM}^{v_t}(s^{w_t}_t,R_t) .$$
It follows that $\|z_t\|_2 \le 2R_X \sum_i v_{t,i} \le 2 R_X\ m \ v_{t,\mathrm{max}}\ \phi_{SLAM}^{v_t}(s^{w_t}_t,R_t)$.\\\\
Combining 1st and 2nd bound for $z_t$, we get $\|z_t\|_2^2 \le 4 R_X^2 m\   v_{t,max}\ \phi_{SLAM}^{v_t}(s^{w_t}_t,R_t)$, for mistake rounds.\\\\
Since, for non-mistake rounds, we have $z_t=0$ and $f_t(w_t)=0$, we get the final inequality.\\\\
\end{proof}
Taking $\max \limits_{t=1}^T\ v_{t,\mathrm{max}} \le v_{max}$, we have the following theorem, which uses the norm bound on $z_t$:
\begin{thm}
\label{theoryboundinperceptron}
Suppose Algorithm~\ref{alg:LA} receives a sequence of instances ${(X_1,R_1),\ldots,(X_T,R_T)}$ and let $R_X$ be the bound on the maximum $\ell_2$ norm of the feature vectors representing the documents. Then the following inequality holds, after optimizing over learning rate $\eta$, $\forall \ w \in \mathbb{R}^d$:
\begin{equation}
\label{eq:perceptronmistakebound}
\begin{split}
\sum_{t=1}^T \text{RML}(s^{w_t}_t,R_t)  \le  \ \sum_{t=1}^T f_t(w) + \sqrt{4\|w\|_2^2 m R_X^2}\sqrt{\sum_{t=1}^T f_t(w)} + \ 4 \|w\|_2^2  m R_X^2 v_{max} \ .
\end{split}
\end{equation}
In the special case where there exists  $w$ s.t. $f_t(w)=0$, $\forall \ t$, we have
\begin{equation}
 \sum_{t=1}^T \text{RML}(s^{w_t}_t,R_t)\ \le \ 4 \|w\|_2^2 m R_X^2 v_{max} .
\end{equation}
\end{thm}

\begin{proof}
The proof follows by plugging in expression for $\|z_t\|_2^2$ (Lemma~\ref{gradientboundinperceptron}) in OGD equation (Prop. OGD Regret), optimizing over $\eta$, using the algebraic trick: $x - b\sqrt{x} -c \le 0 \implies x \le b^2 + c + b\sqrt{c}$ and then using the inequality $f_t(w_t) \ge \text{RML}(s^{w_t}_t,R_t)$. 
\end{proof}
{\bf Note}: The perceptron bound, in Eq.~\ref{eq:perceptronmistakebound}, is a loss bound, i.e., the left hand side is cumulative NDCG/AP induced loss while right side is function of cumulative surrogate loss. We discuss in details the significance of this bound later.

Like perceptron for binary classification, the constant in Eq.~\ref{eq:perceptronmistakebound} needs to be expressed in terms of a ``margin". A natural definition of margin in case of ranking data is as follows: let us assume that there is a linear scoring function parameterized by a unit norm vector $w_\star$, such that all documents for all queries are ranked not only correctly, but correctly with a \emph{margin} $\gamma$:
\begin{equation}
\label{eq:margin}
\begin{aligned}
\min_{t=1}^T \min_{i,j: R_{t,i} > R_{t,j}} w_\star^\top X_{t,i} - w_\star^\top X_{t,j} \geq \gamma .
\end{aligned}
\end{equation}

\begin{cor}
\label{cor:margin}
If the margin condition~\eqref{eq:margin} holds, then total loss, for both NDCG and AP induced loss, is upper bounded by $\tfrac{4 m R_X^2 v_{max}}{\gamma^2}$, a bound independent of the number of instances in the online sequence.
\end{cor}

\begin{proof}

Fix a $t$ and the example $(X_t,R_t)$. Set $w = w_\star/\gamma$. For this $w$, we have
\[
\min_{i,j: R_{t,i} > R_{t,j}} w^\top X_{t,i} - w^\top X_{t,j} > 1 ,
\]
which means that
\[
\min_{i,j: R_{t,i} > R_{t,j}} s^w_{t,i} - s^w_{t,j} > 1
\]
This immediately implies that $\mathbbm{1}(R_{t,i} > R_{t,j})(1 + s^w_{t,j} - s^w_{t,i}) \leq 0$, $\forall \ i,j$ . Therefore, $\phi_{SLAM}^{v_t} (s^{w}_t,R_t)=0$ and hence $f_t(w) = 0$. Since this holds for all $t$, we have $\sum_{t=1}^T f_t(w)=0$.

\end{proof}


\subsubsection{Perceptron Bound-General Discussion} 

We remind once again that $\text{RML}(s^{w_t}_t,R_t)$ is either $1- \text{NDCG}(s^{w_t}_t,R_t)$ or 1$- \text{AP}(s^{w_t}_t,R_t)$, depending on measure of interest. \\

{\bf Importance of learning rate parameter $\eta$}: Like the classification perceptron, Algorithm~\ref{alg:LA} also has the learning rate parameter $\eta$ embedded, and the optimal upper bound on loss is obtained by optimizing over $\eta$. However, unlike classification perceptron, the performance is not independent of $\eta$. The prediction at each round is the ranking obtained from sorted order of score, i.e., $p_t= \argsort (X_t w_t)$.  Let $\mathcal{M}_t$ indicate the rounds, up to time point  $t-1$, where the algorithm did not produce perfect ranking. Starting from $w_1= \vec{0}$, unraveling $w_t$, we get $p_t = \argsort( \sum_{i \in \mathcal{M}_t} - \eta X_t \cdot z_i)$. Now, had $z_i$ been independent of $\eta$, then $p_t$, which is the sorted order of score vector, would have been independent of scaling factor $\eta>0$. However, each $z_i$ is dependent on $w_i$ implicitly (Eq.~\ref{eq:gradientinperceptron}), which themselves are dependent of $\eta$ (recall for classification perceptron, $z_i = - y_i x_i$, i.e., independent of $w_i$ during mistake round $i$). To clarify, we consider, during a mistake round, two score vector $s^1$ and $s^2$, where $s^2= \eta s^1$. Had subgradient $z$, during a mistake round, been indeed independent of $w$ (and hence score $s= X \cdot w$), then $z$ would have been same for both $s^1$ and $s^2$. However, this is not the case. To see this, note that $c_i$ (Eq.~\ref{eq:surrogateinperceptron}), for some $i$, can be 0 for $s^1$ but non-zero for $s^2$, depending on value of $\eta$, which affects the gradient. 

{\bf Dependence of perceptron bound on number of documents per query}: The perceptron bound in Eq.\ \ref{eq:perceptronmistakebound} is meaningful only if $v_{\mathrm{max}}$ is a finite quantity. 

For AP, it can be seen from the definition of $v^{\text{AP}}$ in Eq.~\ref{eq:APweights} that $v_{\mathrm{max}} =1 $.  Thus, for \emph{AP induced loss, the constant in the perceptron bound is}:  $\tfrac{4 m R_X^2 }{\gamma^2}$.

For NDCG, $v_{\mathrm{max}}$ depends on maximum relevance level. Assuming maximum relevance level is finite (in practice, maximum relevance level is usually below $5$), $v_{\mathrm{max}}= O(\log(m))$. Thus, for \emph{NDCG induced loss, the constant in the perceptron bound is}: $\tfrac{4 m \log(m)  R_X^2 }{\gamma^2}$.\\

{\bf Significance of perceptron bound}: The main perceptron bound is given in Eq.~\ref{eq:perceptronmistakebound}, with the special case being captured in Corollary~\ref{cor:margin}. At first glance, the bound might seem non-informative because the left side is the cumulative NDCG/AP induced loss bound, while the right side is a function of the cumulative surrogate loss.
 
The first thing to note is that the perceptron bound is derived from the regret bound in Eq.~\ref{eq:OGDregret}, which is the well-known regret bound of the OGD algorithm applied to an arbitrary convex, Lipschitz surrogate. So, even ignoring the bound in Eq.~\ref{eq:perceptronmistakebound}, the perceptron algorithm is a valid online algorithm, applied to the sequence of convex functions $f_t(\cdot)$, to learn ranking function $w_t$, with a meaningful regret bound. Second, as we had mentioned in the introduction, our perceptron bound is the extension of perceptron bound in classification, to the cumulative NDCG/AP induced losses in the learning to rank setting. This can be observed by noticing the similarity between Eq.~\ref{eq:perceptronmistakebound} and Eq.~\ref{eq:classification-perceptronbound}. In both cases, the the cumulative target loss on the left is bounded by a function of the cumulative surrogate loss on the right, where the surrogate is the hinge (and hinge like SLAM) loss.

The interesting aspects of perceptron loss bound becomes apparent on close investigation of the cumulative surrogate loss term $\sum_{t=1}^T f_t(w)$ and comparing with the regret bound. It is well known that when OGD is run on any convex, Lipschitz surrogate, the guarantee on the regret scales at the rate $O(\sqrt{T})$. So, if we only ran OGD on an arbitrary convex, Lipschitz surrogate, then, even with the assumption of existence of a perfect ranker,  the upper bound on the cumulative loss would have scaled as $O(\sqrt{T})$. However, in the perceptron loss bound, if $\sum_{t=1}^T f_t(w)= o(T^{\alpha})$, then the upper bound on the cumulative loss would scale as $O(T^{\alpha})$, which can be much better than $O(T^{1/2})$ for $\alpha < 1/2$. In the best case of $\sum_{t=1}^T f_t(w)=0$, the total cumulative loss would be bounded, irrespective of the number of instances.

{\bf Comparison and contrast with perceptron for classification}: The perceptron for learning to rank is an extension of the perceptron for classification, both in terms of the algorithm and the loss bound. To obtain the perceptron loss bounds in the learning to rank setting, we had to address multiple non-trivial issues, which do not arise in the classification setting. Unlike in classification, the NDCG/AP losses are not $\{0, 1\}$-valued. The analysis is trivial in classification perceptron since on a mistake round, the absolute value of gradient of hinge loss is 1, which is same as the loss itself. In our setting, Lemma~\ref{gradientboundinperceptron} is crucial, where we exploit the structure of SLAM surrogate to bound the square of gradient by the surrogate loss.

\subsubsection{Perceptron Bound Dependent On NDCG Cut-Off Point}

The bound on the cumulative loss in Eq.~\eqref{eq:perceptronmistakebound} is dependent on $m$, the maximum number of documents per query. It is often the case in learning to rank that though a list has $m$ documents, the focus is on the top $k$ documents ($k \ll m$) in the order sorted by score. The measure used for top-$k$ documents is $\text{NDCG}_k$ (Eq.~\ref{eq:NDCGk}) (there does not exist an equivalent definition for AP). 

We consider a modified set of weights $v^{\text{NDCG}_k}$ s.t. $\phi_{SLAM}^{v^{\text{NDCG}_k}}(s,R) \ge 1- \text{NDCG}_k(s,R)$ holds $\forall \ s$, for every $R$. \emph{We provide the definition of $v^{\text{NDCG}_k}$ later in the proof of Theorem\ref{thm:k-dependence} }.

Overloading notation with $v_t= v^{\text{NDCG}_k}_t$, let $v_{t,max}= \underset{i,j}{\max} \{\dfrac{v_{t,i}}{v_{t,j}}\}$ with $v_{t,i}>$0, $v_{t,j}>$0 and $v_{max} \ge \max_{t=1}^T v_{t,max}$. 

\begin{thm} 
\label{thm:k-dependence}
Suppose the perceptron algorithm receives a sequence of instances ${(X_1,R_1),\ldots,(X_T,R_T)}$. Let $k$ be the cut-off point of NDCG. Also, for any $w \in \mathbb{R}^d$, let $f_t(w)$ be as defined in Eq.~\eqref{eq:functioninperceptron}, but with $\phi_{SLAM}^{v_t}(s^w_t,R_t)=  \phi_{SLAM}^{v^{\text{NDCG}_k}_t}(s^w_t,R_t)$. Then, the following inequality holds, after optimizing over learning rate $\eta$,
\begin{equation}
\label{eq:perceptronmistakeboundfork}
\begin{aligned}
\sum_{t=1}^T (1-\text{NDCG}_k(s^{w_t}_t,R_t)) \le  \sum_{t=1}^T f_t(w) + \sqrt{4 \|w\|_2^2  k R_X^2 v_{\mathrm{max}}}\sqrt{\sum_{t=1}^T f_t(w)} + \ 4 \|w\|_2^2  k R_X^2 v_{\mathrm{max}} .
\end{aligned}
\end{equation}
In the special case where there exists  $w$ s.t. $f_t(w)=0$, $\forall \ t$, we have
\begin{equation}
\sum_{t=1}^T (1-\text{NDCG}_k(s^{w_t}_t,R_t))\  \le \ 4 \|w\|_2^2 k R_X^2 v_{\mathrm{max}} .
\end{equation}
\end{thm}

{\bf Discussion}: Assuming maximum relevance level is finite, we have $v_{\mathrm{max}}= O(\log(k))$ (using definition of $v^{\text{NDCG}_k}$). Thus, the \emph{ constant term in the perceptron bound for $\text{NDCG}_k$ induced loss is}: $4 \|u\|^2  k \log(k) R_X^2 $. This is a significant improvement from original error term, even though the perceptron algorithm is running on queries with $m$ documents, which can be very large. A margin dependent bound can be defined in same way as before.

\begin{proof}

We remind again that ranking performance measures only depend on the permutation of documents and individual relevance level. They do not depend on the identity of the documents. Documents with same relevance level can be considered to be interchangeable, i.e., relevance levels create equivalence classes. Thus, w.l.o.g., we assume that $R_1 \ge R_2 \ge \ldots \ge R_m$ and documents with same relevance level are sorted according to score. Also, $\pi^{-1}(i)$ means position of document $i$ in permutation $\pi$.

We define $v^{\text{NDCG}_k}$ as  
\begin{equation}
\label{eq:truncatedndcgweights}
v^{\text{NDCG}_k}_i = \\
\left\{
	\begin{array}{ll}
		\frac{G(R_i)D(i)}{Z_k(R)}  & \mbox{if } i=1,2,\ldots,k\\
		0  & \mbox{if } i=k+1,\ldots,m .\\ 
           \end{array}
\right.
\end{equation}

We now prove the upper bound property that $\phi_{SLAM}^{v^{\text{NDCG}_k}}(s,R) \ge 1- \text{NDCG}_k(s,R)$ holds $\forall \ s$, for every $R$. We have the following equations:

\begin{equation*}
\begin{aligned}
& \dfrac{\sum_{i=1}^m G(R_i)D(i)\mathbbm{1}(i\le k)}{Z_k(R)}=1 \  \text{and} \ \text{NDCG}_k(s,R)= \dfrac{\sum_{i=1}^m G(R_i)D(\pi_s^{-1}(i))\mathbbm{1}(\pi_s^{-1}(i) \le k)}{Z_k(R)}.\\
& \implies 1- \text{NDCG}_k(s,R) = \dfrac{\sum_{i=1}^m G(R_i)\big(D(i)\mathbbm{1}(i \le k)- D(\pi_s^{-1}(i))\mathbbm{1}(\pi_s^{-1}(i) \le k)\big)}{Z_k(R)}. 
\end{aligned}
\end{equation*}

For $i >k$: $D(i)\mathbbm{1}(i \le k)=$ 0 and since $D(\pi_s^{-1}(i))$ is non-negative, every term in $1- \text{NDCG}_k(s,R)$ is non-positive for $i >k$.

For $i \le k$, there are four possible cases:
\begin{enumerate}

\item $i \ge \pi_s^{-1}(i)$ and $\pi_s^{-1}(i) > k$. This is infeasible since $i \le k$.

\item  $i \ge \pi_s^{-1}(i)$ and $\pi_s^{-1}(i) \le k$. In this case, the numerator in $1- \text{NDCG}_k$ is $G(R_i)(D(i) - D(\pi_s^{-1}(i)))$. Now, since $D(\cdot)$ is a decreasing function, the contribution of the document $i$ to NDCG induced loss is non-positive and can be ignored (since SLAM by definition is sum of positive weighted indicator functions).

\item $i < \pi_s^{-1}(i)$ and $\pi_s^{-1}(i) > k$. In this case, the numerator in $1- \text{NDCG}_k$ is $G(R_i)D(i)$.  Since $ i < \pi_s^{-1}(i)$, that means document $i$ was outscored by a document $j$, where $i < j$ (otherwise, document $i$ would have been put in a position same or above what it is at currently, by $\pi_s$, i.e, $i \ge \pi_s^{-1}(i)$.) Moreover, $R_i >R_j$ (because of the assumption that within same relevance class, scores are sorted). Hence the indicator of SLAM at $i$ would have come on and $v^{\text{NDCG}_k}_i = \frac{G(R_i)(D(i)}{Z_k(R)}$ .

\item $i < \pi_s^{-1}(i)$ and $\pi_s^{-1}(i) \le k$. In this case, the numerator in $1- \text{NDCG}_k$ is $G(R_i)(D(i) - D(\pi_s^{-1}(i)))$. By same reason as c.), the indicator of SLAM at $i$ would have come on and $v^{\text{NDCG}_k}_i > \dfrac{G(R_i)(D(i)- D(\pi_s^{-1}(i)))}{Z_k(R)}$ by definition of $v^{\text{NDCG}_k}$ and the fact that $D(i) > D(i)- D(\pi_s^{-1}(i))$.

\end{enumerate}

Hence, the upper bound property holds.

The proof of Theorem \ref{thm:k-dependence} now follows directly following the argument in the proof of Lemma \ref{gradientboundinperceptron}, by noting a few things:

a) $\sum_{i=1}^k v^{\text{NDCG}_k}_i = 1$. b) $\phi_{SLAM}^{v^{\text{NDCG}_k}}(s,R)$ has same structure as $\phi_{SLAM}^{v^{\text{NDCG}}}(s,R)$ but with different weights. Hence structure of $z_t$ remains same but with weights of $v^{\text{NDCG}_k}$.

Hence, 1st bound on gradient of $z_t$ in proof of Lemma \ref{gradientboundinperceptron} remains same. For the 2nd bound on gradient of $z_t$, the crucial thing that changes is that $\sum_{i=1}^m v^t_i \le k v_{t,\mathrm{max}} v_{t,i'} c^t_{i'}$, with the new definitions of $v_{t,\mathrm{max}}$ according to $v^{\text{NDCG}_k}$. This implies $2R_X \sum v_{t,i} \le 2 R_X\ k \ v_{t,\mathrm{max}}\ \phi^{v_t}(s^{w_t}_t,R_t)$.

\end{proof}

\section{Minimax Bound on Cumulative NDCG/AP Induced Loss}
\label{minimax-bound}
We discuss a lower bound achievable on separable dataset and another perceptron like algorithm achieving the bound.
\subsection{Lower Bound}
The following theorem gives a lower bound on the cumulative NDCG/AP induced loss, achievable by any deterministic online Algorithm.
\begin{thm}
\label{lower-bound}
Suppose the number of documents per query is $m \ge 2$ and relevance vectors are restricted to being binary graded. Let $\mathcal{X} = \{X \in \mathbb{R}^{m \times d}|\  \|X_{j:}\|_2 \le R_X\}$ and $\frac{R^2_X}{\gamma^2} \le d$. Then, for any deterministic online algorithm, there exists a ranking dataset which is separable by margin $\gamma$ (Eq.~\ref{eq:margin}), on which the algorithm suffers $\Omega(\floor{\frac{R^2_X}{\gamma^2}})$ cumulative NDCG/AP induced loss.
\end{thm}


\begin{proof}
Let $T= \floor{\frac{R^2_X}{\gamma^2}}-1$. Since $\frac{R^2_X}{\gamma^2} \le d$, hence $T + 1\le d$ and $(T+1) \gamma^2 \le R_X^2$. Let a ranking dataset consist of the following $T$ document matrices, for $1 \le i \le T$: 
\begin{equation}
X_i = \begin{bmatrix}
R_X \cdot e_{i+1}^\top \\
- R_X \cdot e_{i+1}^\top \\
R_X \cdot e_1^\top \\
\vdots \\
R_X \cdot e_1^\top
\end{bmatrix} \in \mathcal{X},
\end{equation}
where $e_i$ is the unit vector of length $d$ with $1$ in $i$th coordinate and $0$ in others. These document matrices are presented to a deterministic algorithm $\mathcal{A}$ in order.

The relevance vectors for the dataset are set as follows: for matrix $X_i$, if $\mathcal{A}$ puts the 1st document at position $1$ then $R_{i,1}=0, R_{i,2}=1$. Otherwise, $R_{i,1}=1, R_{i,2}=0$. In either case, $R_{i,j} = 0$ for $j > 2$.
With this choice, note that $R_{i,1} > R_{i,2}$ iff $R_{i} = (1,0,0,\ldots,0)^\top$ and $R_{i,1} < R_{i,2}$ iff $R_{i} = (0,1,0,\ldots,0)^\top$.

We have to make sure that, irrespective of what $\mathcal{A}$ does, we can always find a unit norm weight vector $	w_\star$ such that the dataset is actually separable with margin $\gamma$.
Let a ranking function parameter $w_\star \in \mathbb{R}^d$ be defined as follows: $w_{\star,1} = \frac{-\gamma}{2 \cdot R_X}$,
\begin{equation*}
\begin{aligned}
w_{\star,i}= 
\left\{
	\begin{array}{ll}
		{\frac{\gamma}{2 \cdot R_X}}   & \mbox{if } R_{i,1}>R_{i,2}\\
		 {\frac{-\gamma}{2 \cdot R_X}}  & \text{otherwise}\\ 
                     
	\end{array}
\right.
,\quad \text{for } 2\le i\le T+1 .
\end{aligned}
\end{equation*}
For $T+1 < i \le d$, set $w_{\star,i}=0$.
The unit norm condition holds because $\|w_\star\|_2^2= \frac{(T+1) \gamma^2}{4 \cdot R_X^2} \le 1$.

The margin condition holds as follows. Fix $i \in [T]$.
If $R_{i,1}>R_{i,2}$, then $$X_iw_\star = (\gamma/2,-\gamma/2,-\gamma/2,\ldots,-\gamma/2).$$Otherwise, if $R_{i,1}<R_{i,2}$, then
$$X_i w_\star = (-\gamma/2,+\gamma/2,-\gamma/2,\ldots,-\gamma/2).$$
Therefore, in either case, $w_\star$ scores the only relevant document above all irrelevant document by a margin of exactly $\gamma$.

It is clear that, for the above dataset, $\mathcal{A}$ will make a ranking mistake in each round. But we need to argue a bit more: we need to show that the NDCG/AP induced loss per round will be $\Omega(1)$ on each round. Note that a mistake by itself does not guarantee a constant loss incurred since the minimum possible non-zero loss for these loss functions is dependent on $m$.

We have two cases to consider. First, when $\mathcal{A}$ puts document $1$ at the top. Note that, in this case, $R_i = (0,1,0,\ldots,0)^\top$. The least loss $\mathcal{A}$ incurs in such a scenario is when it puts document $2$ in position $2$. Therefore, AP is at most $1/2$ and NDCG is at most $\frac{1/\log_2(1+2)}{1/\log_2(1+1)}$ which means that $1-AP$ and $1-NDCG$ are both $\Omega(1)$. In the second case, $\mathcal{A}$ does not put document $1$ at the top. 
In this case, $R_{i} = (1,0,0,\ldots,0)^\top$ which means that an irrelevant document gets placed at the top. The least loss $\mathcal{A}$ incurs in this scenario is when it puts document $1$ in position $2$. The AP/NDCG induced losses therefore have again the same minimum values in this case as in the previous one. Since the loss incurred in either of the two types of mistakes in $\Omega(1)$, we conclude that
the cumulative NDCG/AP induced loss will be $\Omega(T)= \Omega(\floor{\frac{R_X^2}{\gamma^2}})$. 
\end{proof}

\subsection{Algorithm Achieving Lower Bound}
We will show that the lower bound established in the previous section is actually the minimax bound, achievable by another perceptron type  algorithm. Thus, Algorithm~\ref{alg:LA} is sub-optimal in terms of the bound achieved, since it has a dependence on number of documents per query.

%
%

Our algorithm is inspired by the work of \cite{crammer2002new}. Following their work, we define a new surrogate via a constrained optimization problem for ranking as follows:

\begin{equation}
\label{eq:newlossdef}
\begin{aligned}
\phi_C(s,R) &= \min \delta \\
 \text{s.t.}&  \ \ \ \  \delta \ge  0, \ \ \ \  s_i +\delta \ge  \Delta + s_j, \ \text{if} \  R_i > R_j, \ \ \forall \ i,j.
\end{aligned}
\end{equation}

The above constrained optimization problem can be recast as a hinge-like convex surrogate:
\begin{equation}
\label{eq:newsurrogate}
\phi_C(s, R) = \max_{i \in [m]} \max_{j \in [m]} \ind{R(i)>R(j)} \left( 1 + s_j - s_i \right)_+ .
\end{equation}

The key difference between the above surrogate and the previously proposed SLAM family of surrogates is that the above surrogate does not adapt to different ranking measures. It also does not exhibit the listwise property since
it treats an incorrectly ranked pair in a uniform way independent of where they are placed by the ranking induced by $s$.

Similar to Algorithm~\ref{alg:LA}, we define a sequence of losses $f_t(w)$, defined on a sequence of data $\{X_t,R_t\}_{t \ge 1}$,as follows:

\begin{equation}
\label{eq:functioninnewperceptron}
\begin{aligned}
& f_t(w)= 
\left\{
	\begin{array}{ll}
		\phi_{C} (s^{w}_t,R_t)  & \mbox{if } \text{RML}(s^{w_t}_t,R_t) \neq 0\\
		0  & \mbox{if } \text{RML}(s^{w_t}_t,R_t) = 0\\                       
	\end{array}
\right. \\
\end{aligned}
\end{equation}
Here, $s^{w}_t=X_tw$ and $w_t$ is the parameter produced by Algorithm~\ref{alg:OA} at time $t$, with the adaptive adversary being influenced by the move of perceptron. Note that $f_t(w_t) \ge  \text{RML}(s^{w_t}_t, R_t)$, since, $f_t(w)$ is always non-negative and if $\text{RML}(s^{w_t}_t, R_t)>0$, there is at least one pair of documents whose scores do not agree with their relevances. At that point, the surrogate value becomes greater than $1$.

During a mistake round, the gradient $z$ is calculated as follows: let $i^*,j^*$ be any pair of indices that achieve the max in Eq.~\ref{eq:newsurrogate}. Then,
\begin{equation}
\label{eq:newgradient}
z= \nabla_w \phi_C(s^w,R)= X^{\top} \{ (-e_{i^*} + e_{j^*}) \ind{R(i^*)>R(j^*)}\ind{1 + s_{j^*} -s_{i^*} \ge 0}\} .
\end{equation}

Note that if there are multiple index pairs achieving the max, then an arbitrary subgradient can be written as a convex combination of subgradients computed using each of the pairs.

\floatstyle{ruled}
\newfloat{algorithm}{htbp}{loa}
\floatname{algorithm}{Algorithm}
\begin{algorithm}
\caption{New Perceptron Algorithm Achieving Lower Bound}
\label{alg:OA}
\begin{tabbing}
Learning rate $\eta>0$, $w_1=\mathbf{0} \in \mathbb{R}^d$.\\
For \=$t=1$ to $T$ \\
\> Receive $X_t$ (document list for query $q_t$). \\
\> Set $s^{w_t}_t = X_tw_t$ \ , predicted ranking output $p_t= \argsort(s^{w_t}_t)$.\\
\> Receive $R_t$ \\
\> {\bf If} $\text{RML}(s^{w_t}_t, R_t) \neq 0$ \=$\qquad$// Note:  $\text{RML}(s^{w_t}_t, R_t)=  \text{RML}(\argsort(s^{w_t}_t), R_t)$ \\ 
\> \ \ $w_{t+1} =w_t - \eta z_t$ \>$\qquad$// $z_t$ is defined in Eq.~\eqref{eq:newgradient} \\
\> {\bf else}\\
\> \ \ $w_{t+1} =w_t $\\
{\bf End For}
\end{tabbing}
\end{algorithm}

We have the following loss bound for Algorithm~\ref{alg:OA}.

\begin{thm}
\label{theoryboundinonlinealg}
Suppose Algorithm~\ref{alg:OA} receives a sequence of instances ${(X_1,R_1),\ldots,(X_T,R_T)}$. Let $R_X$ be the bound on the maximum $\ell_2$ norm of the feature vectors representing the documents and $f_t(w)$ be as defined in Eq.~\ref{eq:functioninnewperceptron}. Then the following inequality holds, after optimizing over learning rate $\eta$, $\forall \ w \in \mathbb{R}^d$:
\begin{equation}
\begin{split}
\sum_{t=1}^T \text{RML}(s^{w_t}_t,R_t)  \le  \ \sum_{t=1}^T f_t(w) + 2\|w\|_2 R_X \sqrt{\sum_{t=1}^T f_t(w)} + \ 4 \|w\|_2^2  R_X^2  \ .
\end{split}
\end{equation}
In the special case where there exists  $w$ s.t. $f_t(w)=0$, $\forall \ t$, we have
\begin{equation}
 \sum_{t=1}^T \text{RML}(s^{w_t}_t,R_t)\ \le \ 4 \|w\|_2^2 R_X^2  .
\end{equation}
\end{thm}
\begin{proof}
We first bound the $\ell_2$ norm of the gradient. From Eq.~\ref{eq:newgradient}, we have: 

{\bf 1st bound for $z_t$}:
\begin{equation*}
\begin{aligned}
\|z_t\|_2 \le & \|X_t^{\top}\|_{1 \rightarrow 2} \|\{ (-e_{i^*} + e_{j^*}) \ind{R(i^*)>R(j^*)}\ind{1 + s_{j^*} -s_{i^*} \ge 0}\}\|_1 \le 2R_X.
\end{aligned}
\end{equation*}

{\bf 2nd bound for $z_t$}:

On a mistake round, since there exists at least 1 pair of documents, whose scores and relevance levels are discordant. Hence,  $\phi_C(s^w,R)>1$. Hence, $\|z_t\|_2 \le 2 R_X \le 2 R_X \phi_C(s_t^{w_t},R_t)$.

Thus, $\|z_t\|_2^2 \le 4 R_X^2 \phi_C(s^{w_t}_t,R_t)$. Since $\|z_t\|_2=0$ on non-mistake round, we finally have:

$\|z_t\|_2^2  \le 4 R_X^2 f_t(w_t), \ \forall \ t $. 

The proof then follows as previous:  by plugging in expression for $\|z_t\|^2$ in OGD equation (Prop. OGD Regret), optimizing over $\eta$, using the algebraic trick: $x - b\sqrt{x} -c \le 0 \implies x \le b^2 + c + b\sqrt{c}$ and then using the inequality $f_t(w_t) \ge \text{RML}(s^{w_t}_t,R_t)$. 
%
\end{proof}

As before, we can immediately derive a margin based bound.

\begin{cor}
\label{cor:newmargin}
If the margin condition~\eqref{eq:margin} holds, then total loss, for both NDCG and AP induced loss, is upper bounded by $\tfrac{4 R_X^2 }{\gamma^2}$, a bound independent of the number of instances in the online sequence.
\end{cor}

\begin{proof}

Proof is similar to that of Corollary~\ref{cor:margin}.
\end{proof}

{\bf Importance of learning rate parameter $\eta$}: Algorithm~\ref{alg:OA} also has the learning rate parameter $\eta$ embedded, and the optimal upper bound on loss is obtained by optimizing over $\eta$. However, like classification perceptron, and unlike Algorithm~\ref{alg:LA}, the performance is independent of $\eta$. To see this, we once again use prediction $p_t = \argsort( \sum_{i \in \mathcal{M}_t} - \eta X_t \cdot z_i)$. The prediction is independent of $\eta$ if $z_i$ is independent of $w_i$. Once again, we consider, during a mistake round, two score vector $s^1$ and $s^2$, where $s^2= \eta s^1$. If subgradient $z$, during a mistake round, is indeed independent of $w$ (and hence score $X \cdot w$), then $z$ is same for both $s^1$ and $s^2$. During a mistake round, there is at least one pair of documents, such that $R_i> R_j$, but $s^1_i <s^1_j$. Let us assume that the pair $(i,j)$ obtains the maximum in Eq.~\ref{eq:newsurrogate}. The gradient is as given in Eq.~\ref{eq:newgradient}, with $s$ replaced by $s^1$. However, even for $s^2$, the maximum value in Eq.~\ref{eq:newsurrogate} is obtained for the pair $(i,j)$ (since the differences between any pair of score values are scaled by the same factor $\eta$, going from $s^1$ to $s^2$). Hence, the gradient, in Eq.~\ref{eq:newgradient}, would remain same for $s^1$ and $s^2$.

{\bf Comparison of Algorithm~\ref{alg:LA} and Algorithm~\ref{alg:OA}}: Both of our proposed perceptron-like algorithms can be thought of analogues of the classic perceptron in the learning to rank setting. Algorithm~\ref{alg:OA} achieves the minimax optimal bound on separable datasets, unlike Algorithm~\ref{alg:LA}, whose bound scales with number of documents per query. However, Algorithm~\ref{alg:OA} operates on a surrogate (Eq~\ref{eq:newsurrogate}) which is not listwise in nature, even though it forms an upper bound on the listwise ranking measures. To emphasize, the surrogate does not differentially weigh between errors at different points of the ranked list, which is an important property of popular surrogates in learning to rank. As our empirical results show (Section~\ref{experiments}), on commercial datasets which are not separable, Algorithm~\ref{alg:OA} has significantly worse performance than Algorithm~\ref{alg:LA}.

\section{Related Work in Perceptron for Ranking}
\label{related-work}
There exist a number of papers in the literature dealing with perceptron in the context of ranking. We will compare and contrast our work with existing work, paying special attention to the papers whose setting come closest to ours. 

%
First, we would like to point out that, to the best of our knowledge, there is no work that establishes a number of documents independent bound for NDCG, cut-off at the top $k$ position (Theorem~\ref{thm:k-dependence}). Moreover, we believe our work, for the first time, formally establishes minimax bound, achievable by any deterministic online algorithm, in the learning to rank setting, under the assumption of separability.

\cite{crammer2001} were one of the first  to introduce perceptron in ranking. The setting as well as results of their perceptron are quite different from ours. Their paper assumes there is a fixed set of ranks $\{1,2,\ldots,k\}$. An instance is a vector of the form $x \in \mathbb{R}^d$ and the supervision is one of the $k$ ranks. The perceptron has to learn the correct ranking of $x$, with the loss being $1$ if correct rank is not predicted. The paper does not deal with query-documents list and does not consider learning to rank measures like NDCG/AP. 

The results of \cite{wangsolar} have some similarity to ours. Their paper introduces algorithms for online learning to rank, but does not claim to have any ``perceptron type" results. However, their main theorem (Theorem 2) has a perceptron bound flavor to it, where the cumulative NDCG/AP losses are upper bounded by cumulative surrogate loss and a constant. The major differences with our results are these: \cite{wangsolar} consider a different instance/supervision setting and consequently have a different surrogate loss. It is assumed that for each query $q$, only a pair of documents $(x_i,x_j)$ are received at each online round, with the supervision being $\{+1,-1\}$, depending on whether $x_i$ is more/less relevant than $x_j$. The surrogate loss is defined at pair of documents level, and not at a query-document matrix level. Moreover, there is no equivalent result to our Theorem~\ref{thm:k-dependence}, neither is any kind of minimax bound established. 

The recent work of \cite{jain2015predtron} also contains results similar to ours. One the one hand their predtron algorithm is more general. But on the other hand, the bound achieved by predtron, applied to the ranking case, has a scaling factor $O(m^5)$, significantly worse than our linear scaling. Moreover, it does not have the $\text{NDCG}_k$ bounds scaling as a function of $k$ proved anywhere.

There are other, less related papers; all of which deal with perceptron in ranking, in some form or the other. \cite{ni2008} introduce the concept of margin in a particular setting, with corresponding perceptron bounds. However, their paper does not deal with query-document matrices, nor NDCG/AP induced losses. The works of \cite{elsas2008} and \cite{harrington2003} introduce online perceptron based ranking algorithms, but do not establish theoretical results.  \cite{shen2005} give a perceptron type algorithm with a theoretical guarantee, but in their paper, the supervision is in form of full rankings (instead of relevance vectors). A few recent papers deal with generalization ability of online learning algorithms with pair-wise surrogates \citep{wang2012generalization, kar2013generalization}, online AUC optimization \citep{gao2013one} and optimization at top ranked position \citep{li2014top} However, none of the papers are related to perceptron for learning to rank.
\section{Experiments}
\label{experiments}
We conducted experiments on a simulated dataset and three large scale industrial benchmark datasets. Our results demonstrate the following: 
\begin{itemize}
\item We simulated a margin $\gamma$ separable dataset. On that dataset, the two algorithms (Algorithm~\ref{alg:LA} and Algorithm~\ref{alg:OA}) ranks all but a finite number of instances correctly, which agrees with our theoretical prediction.
\item On three commercial datasets, which are not separable, Algorithm~\ref{alg:LA} shows competitive performance with a strong baseline algorithm, indicating its practical usefulness. Algorithm~\ref{alg:OA} performs quite poorly on two of the datasets, indicating that despite minimax optimality under margin separability, it has limited practical usefulness.
\end{itemize}

{\bf Baseline Algorithm}: We compared our algorithms with the online version of the popular ListNet ranking algorithm \citep{Cao2007}. ListNet is not only one of the most cited ranking algorithms (over 800 citations according to Google Scholar), but also one of the most validated algorithms \citep{tax2015}. 
We conducted online gradient descent on the cross-entropy convex surrogate of ListNet to learn a ranking algorithm in an online manner. While there exists ranking algorithms which have demonstrated better empirical performance than ListNet, they are generally based on non-convex surrogates with non-linear ranking functions. These algorithms cannot be converted in a straight forward way (or not at all) into online algorithms which learn from streaming data. We also did not compare our algorithms with other perceptron algorithms since they do not usually have similar setting to ours and would require modifications. \emph{We emphasize that our objective is not simply to add one more ranking algorithms to the huge variety that already exists. Our experiments on real data are to show that Algorithm~\ref{alg:LA} has competitive performance and has a major advantage over Algorithm~\ref{alg:OA}, due to the difference in the nature of surrogates being used for the two algorithms}.

{\bf Experimental Setting}: For all datasets, we report average $\text{NDCG}_{10}$ and average AP over a time horizon. Average  $\text{NDCG}_{10}$ at iteration $t$ is the cumulative $\text{NDCG}_{10}$ up to iteration $t$, divided by $t$ (same for average AP). We remind that at each iteration $t$, a document matrix is ranked by the algorithm, with the performance (according to $\text{NDCG}_{10}$ or AP) measured against the true relevance vector corresponding to the document matrix. For all the algorithms, the corresponding best learning rate $\eta$ was fixed after conducting experiments with multiple different rates and observing the best time averaged $\text{NDCG}_{10}$/ AP over a fixed time interval.

{\bf Simulated Dataset}:  We simulated a margin separable dataset (Eq.~\eqref{eq:margin}). Each query had $m=20$ documents, each document represented by $20$ dimensional feature vector, and five different relevance level $\{4,3,2,1,0\}$, with relevances distributed uniformly over the documents. The feature vectors of equivalent documents (i.e., documents with same relevance level) were generated from a Gaussian distribution, with documents of different relevance levels generated from different Gaussian distribution. A $20$ dimensional unit norm ranker was generated from a Gaussian distribution, which induced separability with margin. Fig.~\ref{Fig1} compares performance of Algorithm~\ref{alg:LA}, Algorithm~\ref{alg:OA} and online ListNet. The $\text{NDCG}_{10}$ values of the perceptron type algorithms rapidly converge to 1, validating their finite cumulative loss property. To re-iterate, since for separable datasets, cumulative NDCG induced loss is bounded by constant, hence, the time averaged NDCG should rapidly converge to $1$. The OGD algorithm for ListNet has only a regret guarantee of $O(\sqrt{t})$; hence the time averaged regret converges at rate $O(\frac{1}{\sqrt{t}})$, i.e., its convergence is significantly slower than the perceptron-like algorithms.  

\begin{figure}[h]
\begin{center}
\centerline{\includegraphics[height=5cm]{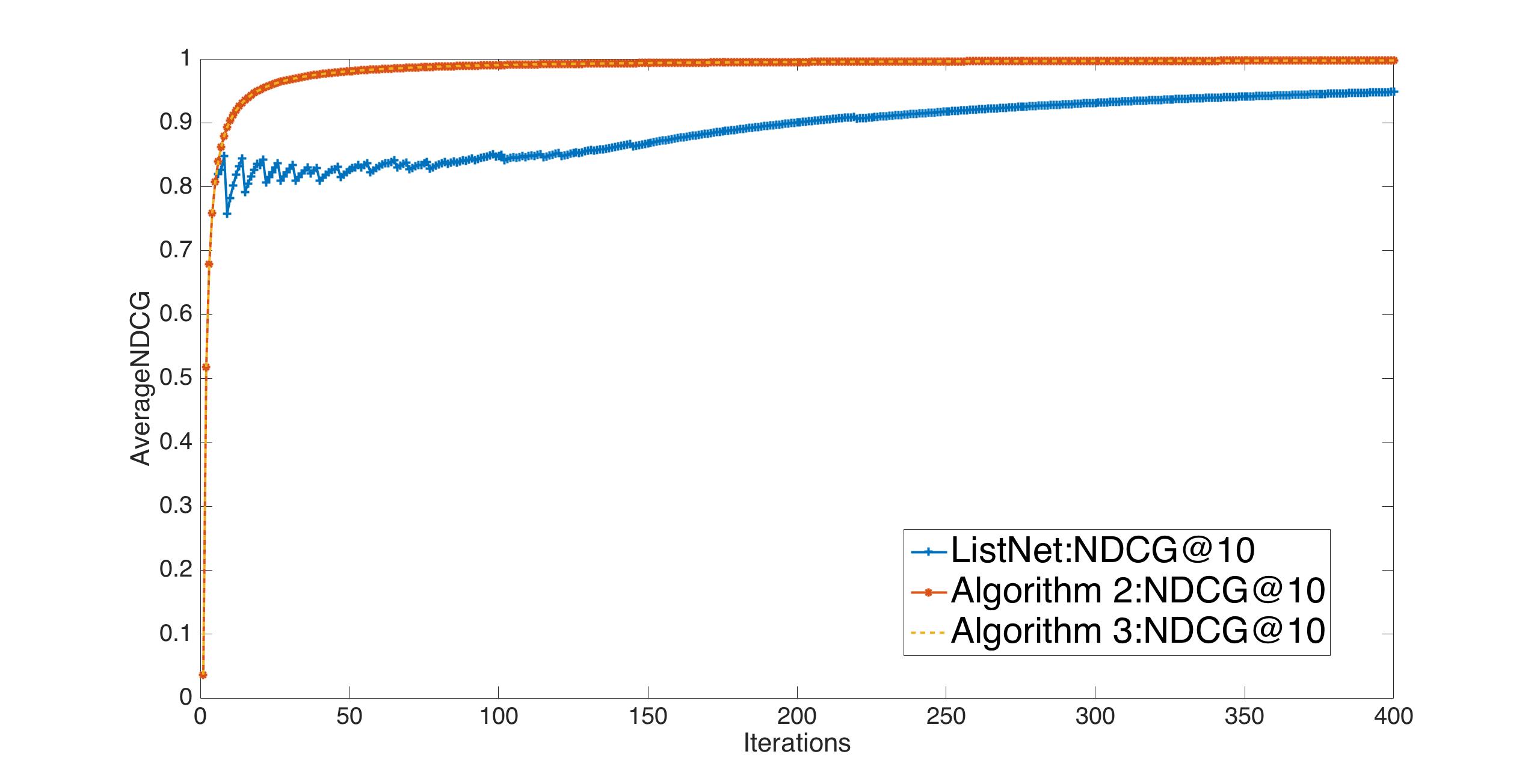}}
\caption{Time averaged $\text{NDCG}_{10}$ for Algorithm~\ref{alg:LA}, Algorithm~\ref{alg:OA} and ListNet, for separable dataset. The two perceptron-like algorithms have imperceptible difference.} \label{Fig1}
\end{center}
\end{figure} 

\begin{figure}[ht!]
     \begin{center}
     \subfigure[MSLR-WEB10K ]{%
            \label{Fig3}
           \includegraphics[height=5cm]{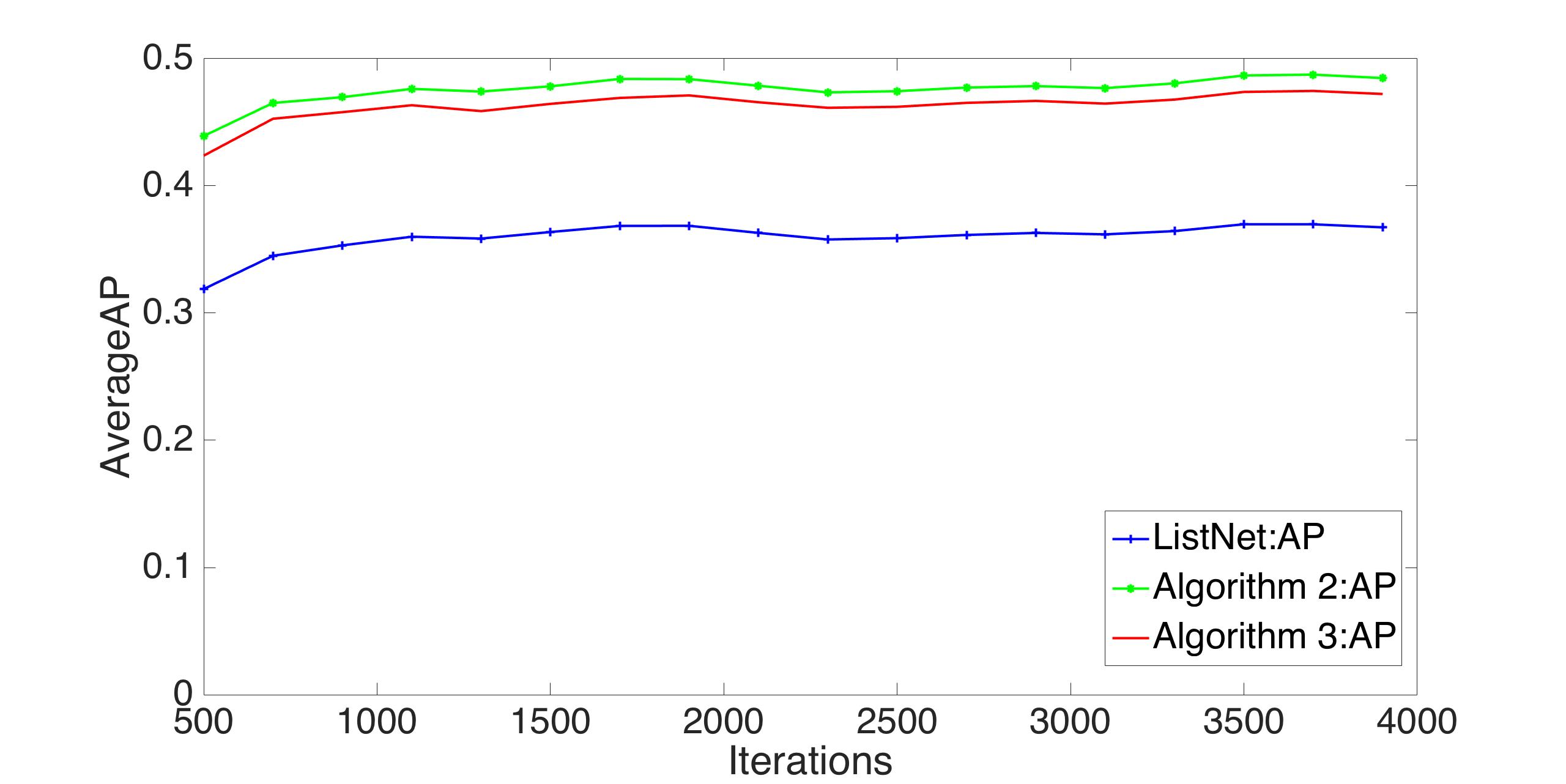}
        }\\
      \subfigure[Yahoo]{%
            \label{Fig4}
            \includegraphics[height=5cm]{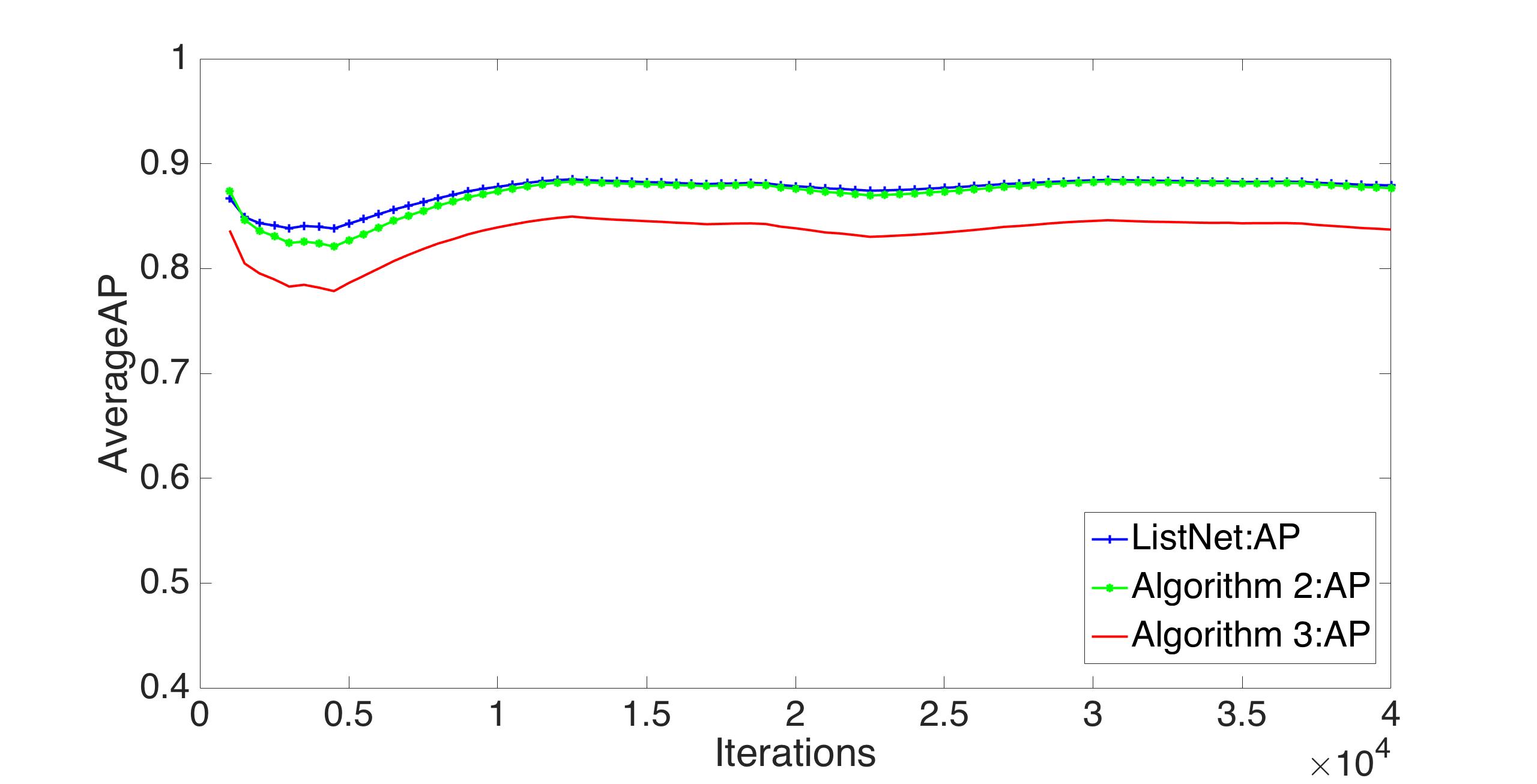}
        }\\
          \subfigure[Yandex]{%
            \label{Fig5}
            \includegraphics[height=5cm]{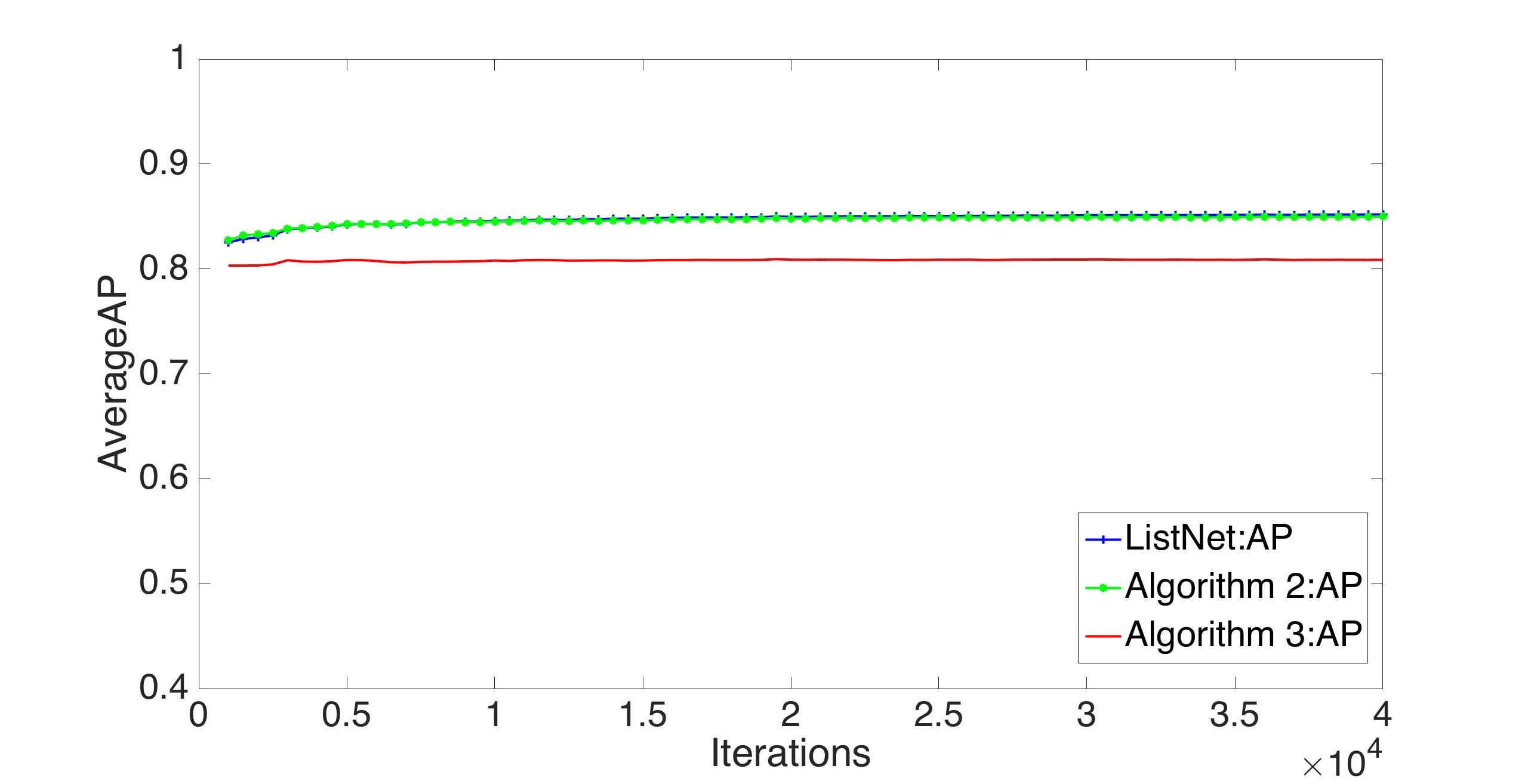}
        }%
      \end{center}
    \caption{Time averaged AP for Algorithm~\ref{alg:LA}, Algorithm~\ref{alg:OA} and ListNet for 3 commercial datasets.}%
   \label{fig:subfiguresAP}
\end{figure}

\begin{figure}[ht!]
     \begin{center}
     \subfigure[MSLR-WEB10K ]{%
            \label{Fig6}
           \includegraphics[height=5cm]{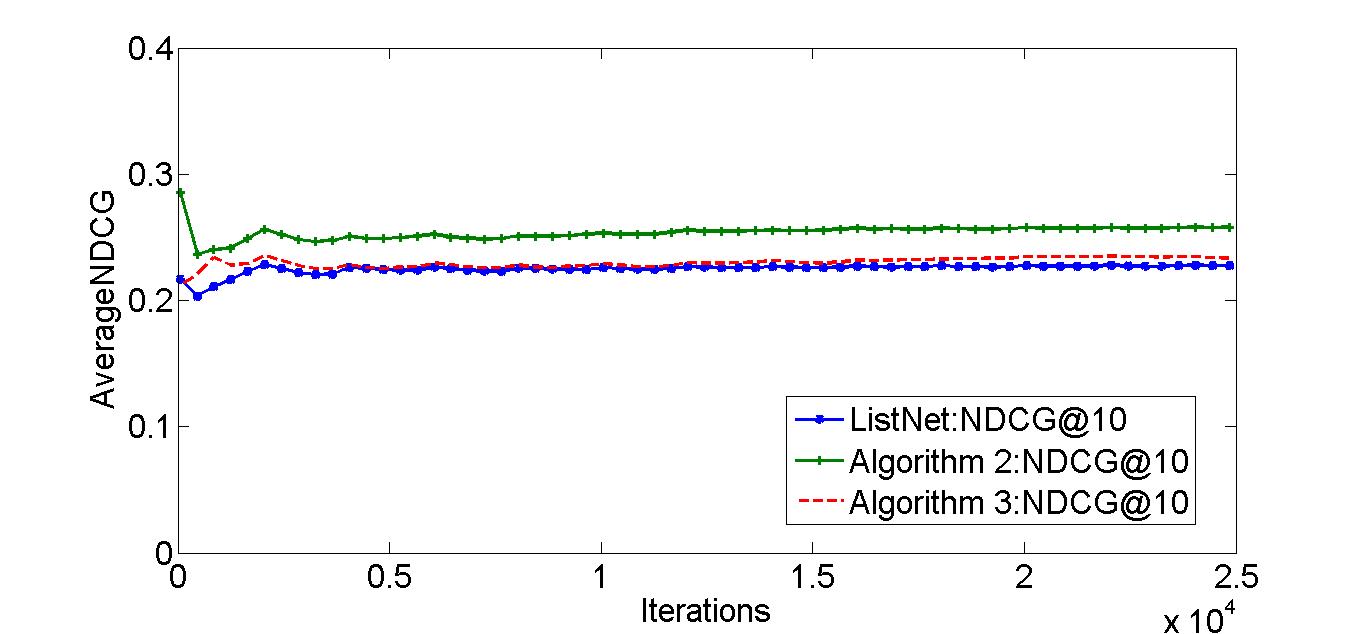}
        }\\
      \subfigure[Yahoo]{%
            \label{Fig7}
            \includegraphics[height=5cm]{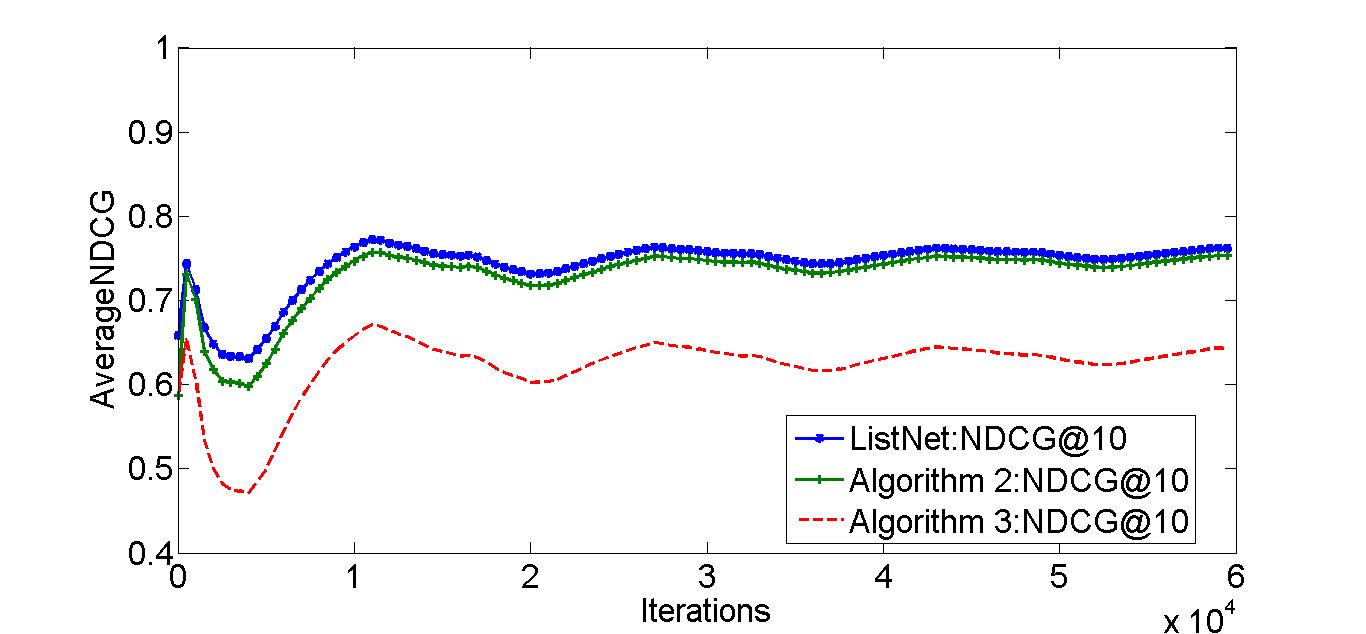}
        }\\
          \subfigure[Yandex]{%
            \label{Fig8}
            \includegraphics[height=5cm]{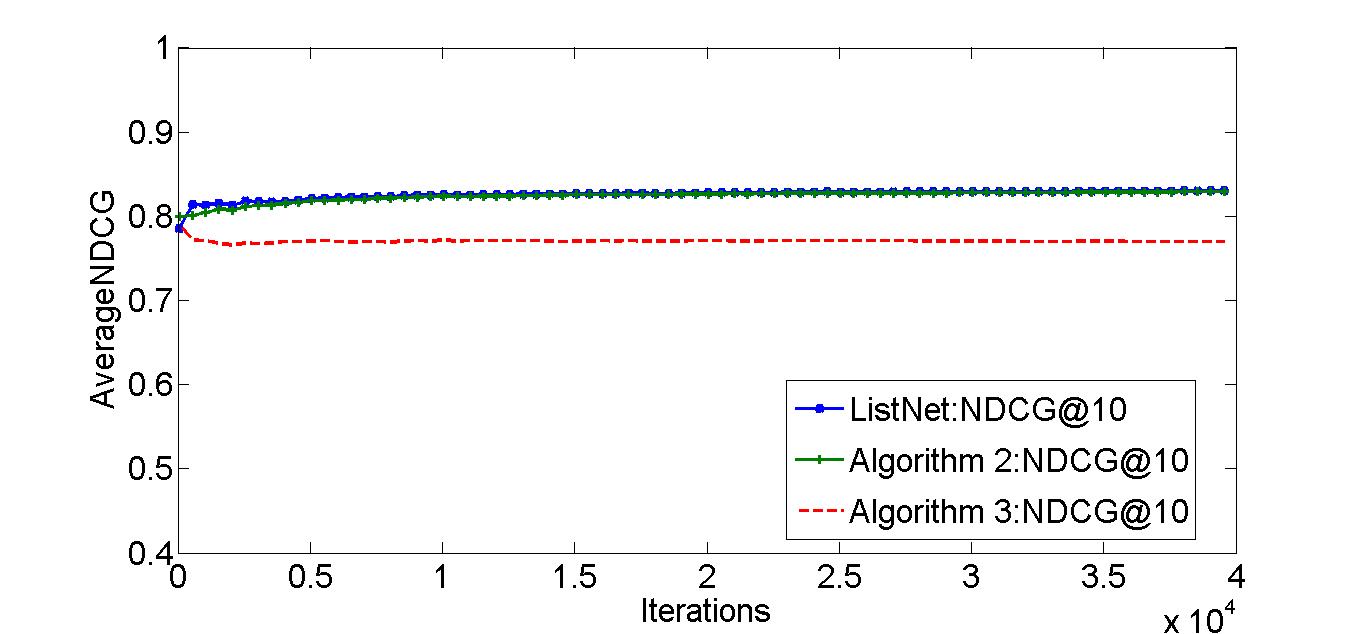}
        }%
      \end{center}
    \caption{Time averaged $\text{NDCG}_{10}$ for Algorithm~\ref{alg:LA}, Algorithm~\ref{alg:OA} and ListNet for 3 commercial datasets.}%
   \label{fig:subfiguresNDCG}
\end{figure}

%
%
%
%

{\bf Commercial Datasets}: We chose three large scale ranking datasets released by the industry to analyze the performance of our algorithms. MSLR-WEB10K \citep{liu2007} is the dataset published by Microsoft's Bing team, consisting of $10,000$ unique queries, with feature dimension of size $245$ and $5$ distinct relevance levels. Yahoo Learning to Rank Challenge dataset \citep{chapelle2011learning} consists of 19,944  unique queries, with feature dimension of size $700$ and 5 distinct relevance levels. Yandex, Russia's biggest search engine, published a dataset (link to the dataset given in the work of~\cite{chapelle2011learning}) consisting of $9124$ queries, with feature dimension of size $245$ and 5 distinct relevance levels. Since AP is suited to binary relevance vectors, we converted the multi-graded vectors to binary vectors when comparing algorithms based on AP. All documents with non-zero relevance grade were considered relevant for the purpose of conversion.

Algorithm~\ref{alg:LA} performs better than ListNet on MSLR-WEB dataset (average NDCG@10 over last ten iterations= 0.25 vs 0.22, average AP over last 10 iterations= 0.49 vs 0.37), performs slightly worse on Yahoo dataset (average NDCG@10 over last ten iterations= 0.75 vs 0.74, average AP over last ten iterations = 0.875 vs 0.87) and has overlapping performance on Yandex dataset. See Fig.~\ref{fig:subfiguresAP} and Fig.~\ref{fig:subfiguresNDCG} for AP and NDCG results respectively. The experiments validate that our proposed perceptron type algorithm (Algorithm~\ref{alg:LA}) has competitive performance compared to online ListNet on real ranking datasets, even though it does not achieve the theoretical lower bound. Algorithm~\ref{alg:OA} performs quite poorly on both Yandex and Yahoo datasets. One possible reason for the poor performance is that the underlying surrogate (Eq.~\ref{eq:newsurrogate}) is not listwise in nature. It does not put more emphasis on errors at the top and hence, is not very suitable for a listwise ranking measure like NDCG, even though it achieves the theoretical lower bound on separable datasets.



\section{Conclusion}
We proposed two perceptron-like algorithms for learning to rank, as analogues of the perceptron for classification algorithm. We showed how, under assumption of separability (i.e., existence of a perfect ranker), the cumulative NDCG/AP induced loss is bounded by a constant. The first algorithm operates on a listwise, large margin family of surrogates, which are adaptable to NDCG and AP. The second algorithm is based on another large margin surrogate, which does not have the listwise property. We also proved a lower bound on cumulative NDCG/AP loss under a separability condition and showed that it is the minimax bound, since our second algorithm achieves the bound. We conducted experiments on simulated and commercial datasets to corroborate our theoretical results. 

An important aspect of perceptron type algorithms is that the ranking function is updated only on a mistake round. Since non-linear ranking functions are generally have better performance than linear ranking functions, an online algorithm learning a flexible non-linear kernel ranking function would be very useful in practice.  We highlight how perceptron's ``update only on mistake round" aspect can prove to be powerful when learning a non-linear kernel ranking function. 
Since the score of each document is obtained via inner product of ranking parameter $w$ and feature representation of document $x$, this can be easily kernelized to learn a range of non-linear ranking functions.  However, the inherent difficulty of applying OGD to a convex ranking surrogate with kernel function is that at each update step, the document list ($X$ matrix) will need to be stored in memory. For moderately large dataset, this soon becomes a practical impossibility. One way of bypassing the problem is to approximately represent the kernel function via an explicit feature projection \citep{rahimi2007,le2013}. However, even for moderate length features (like 136 for MSWEB10K), the projection dimension becomes too high for efficient computation. Another technique is to have a finite budget for storing document matrices and discard carefully chosen members from the budget when budget capacity is exceeded. This budget extension has been studied for perceptron in classification \citep{dekel2008,cavallanti2007}. The fact that perceptron updates are only on mistake rounds leads to strong theoretical bounds on target loss. For OGD on general convex surrogates, the fact that function update happens on every round leads to inherent difficulties when using their kernelized versions \citep{zhao2012} (the theoretical guarantees on the target loss are not as strong as in the kernelized perceptron on a budget case).  The results presented in this paper open up a fruitful direction for further research: namely, to extend the perceptron algorithm to non-linear ranking functions by using kernels and establishing theoretical performance bounds in the presence of a memory budget. 

\acks{We gratefully acknowledge the support of NSF under grant IIS-1319810. We also thank Prateek Jain for pointing out the relevant question on perceptron bound for NDCG cut-off at $k \ll m$.}

\bibliography{Ranking}
\bibliographystyle{plainnat}

\newpage
\appendix
\section{\\Proof of Theorem.\ref{eq:upperbound1}}
 \label{App:AppendixA}

{\bf Proof for AP}:
As stated previously, documents pertaining to every query are sorted according to relevance labels. We point out another critical property of AP (for that matter any ranking measure). AP is only affected when scores of 2 documents, which have different relevance levels, are not consistent with the relevance levels, as long as ranking is obtained by sorting scores in descending order. That is, if $R_i=R_j$, then it does not matter whether $s_i> s_j$ or $s_i < s_j$. So, without loss of generality, we can always assume that within same relevance class, the documents are sorted according to scores. That is, if $R_i=R_j$ with $i<j$, then $s_i \ge s_j$. The without loss of generality holds because SLAM is calculated with knowledge of relevance and score vector. Thus, within same relevance class, we can sort the documents according to their scores (effectively exchanging document identities), without affecting SLAM loss.

Let $R \in \mathbb{R}^m$ be an arbitrary binary relevance vector, with $r$ relevant documents and $m-r$ irrelevant documents in a list. AP loss is only incurred if at least 1 irrelevant document is placed above at least 1 relevant document. With reference to $\phi^v_{SLAM}$ in Eq.~\eqref{eq:theoreticalloss}, for any $i \ge r+1$ and $\forall \ j>i$,  we have $\mathbbm{1}(R_i>R_j)=0$, since $R_i= R_j=0$. For any $i \ge r+1$ and $\forall \ j<i$, $\mathbbm{1}(R_i>R_j)=0$ since documents are sorted according to relevance labels and $R_i=0, R_j=1$. Thus, w.l.o.g., we can take $v_{r+1},..., v_m=0$, since indicator in SLAM loss will never turn on for $i \ge r+1$.

Let a score vector $s$ be such that an irrelevant document $j$  has the highest score among $m$ documents. Then, $\phi^v_{SLAM}= v_1(1 +s_j-s_1) +v_2(1+s_j-s_2) +...+v_r(1+s_j-s_r)$.  The maximum possible AP induced loss in case at least one irrelevant document has higher score than all relevant documents is when all irrelevant documents outscore all relevant documents. The AP loss in that case is: $ 1- \frac{1}{r}(\frac{1}{m-r+1} + \frac{2}{m-r+2}+..+ \frac{r}{m-r+r})$. Since $\phi^v_{SLAM}$ has to upper bound AP $\forall s$ (for each $R$) and since $s_j$ can be infinitesimally greater than all other score components (thus, $1+s_j-s_i \sim 1, \ \forall \ i=1,\ldots,r$), we need the following equation for upper bound property to hold:

$v_1 + v_2+...+ v_r \ge  1- \frac{1}{r}(\frac{1}{m-r+1} + \frac{2}{m-r+2}+..+ \frac{r}{m-r+r})$.

Similarly, let a score vector $s$ be such that an irrelevant document $j$ has higher score than all but the 1st relevant document. Then $\phi^v_{SLAM}= v_2(1 +s_j-s_2) +v_3(1+s_j-s_3) +...+v_r(1+s_j-s_r)$. The maximum possible AP induced loss in this case occurs when all irrelevant documents are placed above all relevant documents except the first relevant document. The AP loss in that case is: $ 1- \frac{1}{r}(1+ \frac{2}{m-r+2} + \frac{3}{m-r+3}+..+ \frac{r}{m-r+r})$. Following same line of logic for upper bounding as before, we get

$v_2 + v_3+...+ v_r \ge  1- \frac{1}{r}(1+ \frac{2}{m-r+2} + \frac{3}{m-r+3}+..+ \frac{r}{m-r+r})$.

Likewise, if we keep repeating the logic, we get sequence of inequalities, with the last inequality being

$v_r \ge 1 - \frac{1}{r}(r-1 + \frac{r}{m-r+r})$.

Now, it can be easily seen that our definition of $v^{\text{AP}}$ satisfies the inequalities.

{\bf Proof for NDCG}:

We once again remind that $\pi^{-1}(i)$ means position of document $i$ in permutation $\pi$. Thus, if document $i$ is placed at position $j$ in $\pi$, then $\pi^{-1}(i)=j$. Moreover, like AP, we assume that $R_1 \ge R_2 \ge \ldots \ge R_m$ and that within same relevance class, documents are sorted according to score. We have a modified definition of NDCG, for $k=m$, which is required for the proof:
\begin{equation}
\label{eq:NDCG}
\begin{split}
\text{NDCG}(s,R)= \frac{1}{Z(R)}\sum_{i=1}^m G(R_i)D(\pi^{-1}_s(i)) 
\end{split}
\end{equation}
where $G(r)= 2^r -1$, $D(i)= \frac{1}{\log_2{(i+1)}}$, $Z(R)= \underset {\pi}{\max}\sum_{i=1}^m G(R_i)D(\pi^{-1}(i))$.\\
We begin the proof:
\begin{equation*}
\begin{split}
&\quad 1- \text{NDCG}(s,R)  \\
&= \frac{1}{Z(R)} \sum_{i=1}^{m} G(R_i) D(i) - \frac{1}{Z(R)}\sum_{i=1}^m G(R_i)D(\pi^{-1}_s(i)) \\
& = \frac{1}{Z(R)} \sum_{i=1}^{m} G(R_i) \left( D(i) - D(\pi_s^{-1}(i)) \right) 
\end{split}
\end{equation*}

Now, $D(i)= \frac{1}{\log _2 (1+i)}$ is a decreasing function of $i$. $D(i) - D(\pi_s^{-1}(i))$ is positive only if $i < \pi_s^{-1}(i)$. This means that document $i$ in the original list,  is placed at position $\pi_s^{-1}(i)$, which comes after $i$, by sorted order of score vector $s$. By the assumption that indices of documents within same relevance class are sorted according to their scores, this means that document $i$ is outscored by another document (say with index $k$) with lower relevance level. At that point, the function $\max(0,\underset{j=1,\ldots,m}{ \max}\{\mathbbm{1}(R_i>R_j)(1+ s_j -s_i)\})$ turns on with value at least $1$ (i.e.,  $(1+s_k -s_i >1)$) and with weight vector $v^{\text{NDCG}}_i = \dfrac{G(R_i)D(i)}{Z(R)}$. We can now easily see the upper bound property.\\\\

\end{document}